\pdfoutput=1
\documentclass[11pt]{article}
\usepackage{fullpage}

\def\preprint{}

\usepackage{dsfont}
\usepackage{amsmath}
\usepackage{amsfonts}
\usepackage{amsthm}
\usepackage[algo2e,ruled]{algorithm2e}
\usepackage[pagebackref,colorlinks=true,linkcolor=blue,citecolor=blue]{hyperref}
\usepackage[dvipsnames]{xcolor}

\usepackage{natbib}
\newtheorem{theorem}{Theorem}
\newtheorem{definition}[theorem]{Definition}
\newtheorem{lemma}[theorem]{Lemma}
\newtheorem{conjecture}[theorem]{Conjecture}
\newtheorem{remark}[theorem]{Remark}

\newcommand{\cD}{\ensuremath{{\mathcal D}}}
\newcommand{\cP}{\ensuremath{{\mathcal P}}}
\newcommand{\cA}{\ensuremath{{\mathcal A}}}
\DeclareMathOperator{\bE}{\mathbb E}
\newcommand{\ber}{\ensuremath{\mathrm{Ber}}}
\newcommand{\adv}{\ensuremath{\mathrm{adv}}}
\newcommand{\mcerror}{\ensuremath{\mathrm{mc}\textup{-}\mathrm{error}}}
\newcommand{\cF}{\ensuremath{{\mathcal F}}}

\newcommand{\cQ}{\ensuremath{{\mathcal Q}}}
\newcommand{\bR}{\ensuremath{\mathbb{R}}}
\newcommand{\bZ}{\ensuremath{\mathbb{Z}}}

\newcommand{\listl}{\textnormal{\textsf{ListL}}}
\newcommand{\bool}{\ensuremath{\mathsf{B}}}
\newcommand{\parity}{\ensuremath{\mathsf{BP}}}
\newcommand{\anti}{\ensuremath{\mathsf{BA}}}

\newcommand{\fat}{\ensuremath{\mathsf{fat}}}

\newcommand{\cH}{\ensuremath{\mathcal{H}}}

\mathchardef\mhyphen="2D
\newcommand{\oi}{\ensuremath{\mathsf{OI}}}
\newcommand{\dfr}{\ensuremath{\mathsf{DF\mhyphen R}}}
\newcommand{\dfa}{\ensuremath{\mathsf{DF\mhyphen A}}}
\newcommand{\dsr}{\ensuremath{\mathsf{DS\mhyphen R}}}
\newcommand{\dsa}{\ensuremath{\mathsf{DS\mhyphen A}}}
\newcommand{\sampdfr}{\ensuremath{\mathsf{SAMP\mhyphen DF\mhyphen R}}}
\newcommand{\sampdfa}{\ensuremath{\mathsf{SAMP\mhyphen DF\mhyphen A}}}
\newcommand{\sampdsr}{\ensuremath{\mathsf{SAMP\mhyphen DS\mhyphen R}}}
\newcommand{\sampdsa}{\ensuremath{\mathsf{SAMP\mhyphen DS\mhyphen A}}}

\newcommand{\loss}{\ensuremath{\mathsf{loss}}}

\newcommand{\vect}{\ensuremath{\mathsf{vec}}}

\newcommand{\acc}{\ensuremath{\textnormal{ACCEPT}}}
\newcommand{\rej}{\ensuremath{\textnormal{REJECT}}}

\newcommand{\cC}{\ensuremath{\mathcal{C}}}
\newcommand{\nei}{\ensuremath{\mathsf{Nei}}}

\newcommand{\one}{\mathds{1}}

\newcommand{\lemmaref}[1]{Lemma~\ref{#1}}
\newcommand{\sectionref}[1]{Section~\ref{#1}}
\newcommand{\algorithmref}[1]{Algorithm~\ref{#1}}
\newcommand{\appendixref}[1]{Appendix~\ref{#1}}
\newcommand{\conjectureref}[1]{Conjecture~\ref{#1}}
\newcommand{\theoremref}[1]{Theorem~\ref{#1}}
\newcommand{\remarkref}[1]{Remark~\ref{#1}}
\newcommand{\definitionref}[1]{Definition~\ref{#1}}

\title{Metric Entropy Duality and the Sample Complexity of Outcome Indistinguishability}
\author{Lunjia Hu\thanks{Computer Science Department, Stanford University. 
Email: \texttt{lunjia@stanford.edu}. 
Supported by Moses Charikar’s Simons Investigator award and Omer Reingold's NSF Award IIS-1908774.}
\and 
Charlotte Peale\thanks{Computer Science Department, Stanford University. 
Email: \texttt{cpeale@stanford.edu}.
Supported by the Simons Foundation Collaboration on the Theory of Algorithmic Fairness.} 
\and 
Omer Reingold\thanks{Computer Science Department, Stanford University. 
Email: \texttt{reingold@stanford.edu}.
Supported by the Simons Foundation Collaboration on the Theory of Algorithmic Fairness, the Sloan Foundation Grant 2020-13941 and the Simons Foundation investigators award 689988.}}
\date{}
\begin{document}
\maketitle
\begin{abstract}%
We give the first sample complexity characterizations for outcome indistinguishability, a theoretical framework of machine learning recently introduced by Dwork, Kim, Reingold, Rothblum, and Yona (STOC 2021). In outcome indistinguishability, the goal of the learner is to output a predictor that cannot be distinguished from the target predictor by a class $\cD$ of distinguishers examining the outcomes generated according to the predictors' predictions. While outcome indistinguishability originated from the algorithmic fairness literature, it provides a flexible objective for machine learning even when fairness is not a consideration. In this work, we view outcome indistinguishability as a relaxation of PAC learning that allows us to achieve meaningful performance guarantees under data constraint.

In the distribution-specific and realizable setting where the learner is given the data distribution together with a predictor class $\cP$ containing the target predictor, we show that the sample complexity of outcome indistinguishability is characterized by the metric entropy of $\cP$ w.r.t.\ the dual Minkowski norm defined by $\cD$, and equivalently by the metric entropy of $\cD$ w.r.t.\ the dual Minkowski norm defined by $\cP$. This equivalence makes an intriguing connection to the long-standing metric entropy duality conjecture in convex geometry. Our sample complexity characterization implies a variant of metric entropy duality, which we show is nearly tight. In the distribution-free setting, we focus on the case considered by Dwork et al.\ where $\cP$ contains all possible predictors, hence the sample complexity only depends on $\cD$. In this setting, we show that the sample complexity of outcome indistinguishability is characterized by the fat-shattering dimension of $\cD$.

We also show a strong sample complexity separation between realizable and agnostic outcome indistinguishability in both the distribution-free and the distribution-specific settings. This is in contrast to distribution-free (resp.\ distribution-specific) PAC learning where the sample complexity in both the realizable and the agnostic settings can be characterized by the VC dimension (resp.\ metric entropy).
\end{abstract}

\ifdefined\preprint
\newpage
\else
\begin{keywords}%
  outcome indistinguishability, covering number, duality conjecture, fat shattering%
\end{keywords}
\fi

\section{Introduction}
Prediction algorithms based on machine learning are becoming increasingly influential on major decisions that affect individual's lives in settings such as medical diagnoses, loan applications, and educational admissions processes. As a result, we must be careful that the predictions of these algorithms do not discriminate against any sub-communities within the input population. Unfortunately, standard measures of prediction accuracy don't necessarily guarantee the absence of such discrimination. Consider a binary classifier that is used to predict an individual's probability of repaying a loan. A natural way to measure the success of our classifier would be to use classification error--the fraction of instances from some representative test set that it classifies incorrectly. Ideally, we'd hope that a classifier with $5\%$ error incorrectly classifies any particular individual in our population only $5\%$ of the time. However, if some low-income minority group makes up 10\% of the population and has a low probability (say 40\%) of repaying a loan on average, a classifier that chooses to focus on getting $99\%$ accuracy on the majority group and applies a blanket policy of classifying every individual in the minority group as unlikely to pay back the loan can still receive this $5\%$ classification error guarantee despite incorrectly classifying minority individuals $40\%$ of the time. If factored into loan application decisions, such a classifier could deny loans to many worthy applicants from the minority group, further cementing and potentially exacerbating the financial disparities present in the population. 

This potential for discrimination was the original inspiration behind \emph{Outcome indistinguishability} (OI), a theoretical framework for machine learning recently proposed by~\citet*{dwork2021outcome} that aims to address disparities in the treatment of different groups by replacing the accuracy objective with a more flexible objective that can give stronger guarantees for sub-communities in a population. OI is a framework used when learning predictors, rather than classifiers. 
In this setting, an analogous measure to the classification error of a learned predictor $p$ is the $\ell_1$ error considered by e.g.\ \citet{MR1408000}.
Instead of using the $\ell_1$ error as a quality measure,
the main idea of OI is to measure the quality of a learned predictor $p$ by how easy it is for a predetermined class $\cD$ of distinguishers to distinguish between the output of $p$ and the output of the target predictor.

More precisely, the goal of OI is to output a predictor $p:X\rightarrow [0,1]$ assigning probabilities $p(x)\in[0,1]$ to individuals $x\in X$.
Given a distribution $\mu$ over $X$, every predictor $p$ defines a distribution $\mu_p$ for individual-outcome pairs $(x,o)\in X\times \{0,1\}$ where the individual $x$ is drawn from $\mu$, and the outcome $o\sim \ber(p(x))$ is drawn from the Bernoulli distribution with mean $p(x)$.
The input to the OI learner consists of examples $(x_1,o_1),\ldots,(x_n,o_n)$ drawn i.i.d.\ from $\mu_{p^*}$ for an unknown target predictor $p^*$. The outputted predictor is audited by a set of distinguishers $\cD$. Here, a distinguisher takes an individual-outcome pair $(x,o)\in X\times\{0,1\}$ as input, and either accepts or rejects the input.\footnote{
\citet{dwork2021outcome} also considered more advanced distinguishers with access to additional information about the predictors such as the predictions themselves or the predictor code, but the simplest distinguishers we describe here can already express non-trivial objectives.
Also,
while we describe outcome indistinguishability in the binary outcome setting, it is possible to generalize the notion to an arbitrary number of possible outcomes by considering predictors that output a description of a more general distribution.
} 
The distinguishing advantage of a distinguisher $d$ on two predictors $p_1$ and $p_2$ is defined as the absolute difference between the acceptance probabilities of $d$ on $(x,o)$ drawn from $\mu_{p_1}$ and $\mu_{p_2}$.
The quality of the learned predictor $p$ is then measured by the maximum distinguishing advantage (MDA) on $p$ and $p^*$ over all distinguishers $d\in\cD$.
In the loan repayment setting described above, adding a distinguisher that only compares the rate of positive classification on the minority group could identify and prevent this disparity for a small enough allowable MDA threshold. 

It can be shown that the MDA never exceeds the $\ell_1$ error $\bE_{x\sim\mu}|p(x) - p^*(x)|$, and when $\cD$ contains all possible distinguishers, the MDA and the $\ell_1$ error are known to be equal (see \lemmaref{lm:advantage-l1}). 
However, by leveraging our ability to select only the distinguishers we care about in OI, we can tune the MDA to be a more suitable quality measure compared to $\ell_1$ error even in settings outside of fairness. As an example, suppose we would like to learn a binary image classifier that can be used in self-driving cars to determine whether the road contains an obstruction. Ideally we would like to learn a model that gives classification error very close to zero because it means that we can expect the car to fail to detect an obstruction extremely rarely. However, what if we have insufficient data to learn a classifier with extremely low error, and must settle for larger error on the order of 1\% or more? This is where we need to observe that not all errors are born equal. Failing to recognize leaves on the road is dramatically different from failing to identify a parking vehicle. Unfortunately, a 1\% error may be completely concentrated on important obstructions (that may occur less than 1\% of the time). An overall error rate that would guarantee minuscule errors in important cases may very well be impossible and concentrating on the errors we care about may be mandated.

This example demonstrates that classification error alone may not be enough to tell us whether this is a model we would trust on the roads. In particular, the ability to refine our measure of performance to focus on the particular types of mistakes that we care about most would give us a better understanding of performance and might potentially allow us to get good targeted accuracy guarantees even with insufficient data to achieve high accuracy against \emph{all} types of errors.
If we believe that distinguishing between instances where a large obstruction is or is not present requires a smaller number of circuit gates than distinguishing between instances that may contain more minor obstructions such as a small tree branch or plastic bag, choosing $\cD$ to contain all distinguishers with small circuit sizes would filter out serious errors such as missing a pedestrian crossing the street even when we cannot achieve any meaningful overall accuracy due to limited data.
Moreover, if $\cD$ contains a distinguisher that accepts $(x,o)$ if and only if $x$ contains a large obstruction and $o$ equals $1$, a predictor that significantly underestimates the true prediction $p^*(x)$ when $x$ contains a serious obstruction would have a high MDA. 
In general,
when the distinguisher class $\cD$ is properly chosen, a predictor with $\ell_1$ error and MDA both being about $0.1$ can be less preferable than a predictor with $\ell_1$ error being $0.2$ but MDA being only $0.01$.

\subsection{Sample Complexity Characterizations}
\label{sec:characterization}
Traditionally, outcome indistinguishability and related notions such as multicalibration \citep*{hebert2018multicalibration} have been viewed as providing stronger guarantees than $\ell_1$ error alone by allowing a predictor's performance to be fine-tuned at the group level rather than just looking at the entire population. However, our obstruction-identification example from the previous section demonstrates how OI can also be viewed as a useful \emph{relaxation} of the standard $\ell_1$ performance benchmark. By focusing on the important errors specified by the distinguisher class,
outcome indistinguishability may allow us to achieve good performance with relatively small sample size.
It is natural to ask:
how much improvement in the sample size do we get from OI?
This is the main focus of this paper---understanding the sample complexity of outcome indistinguishability.

It is one of the major objectives of learning theory to
understand
the sample complexity of various learning tasks and many influential characterizations of sample complexity have been discovered. The most notable example is the VC dimension for PAC learning \citep*{MR3408730,MR1072253,MR1088804}. 
In PAC learning \citep*{valiant1984theory,MR776268}, 
the learner receives examples $(x_1,h^*(x_1)),\ldots,(x_n,h^*(x_n))$ where $x_1,\ldots,x_n$ are drawn i.i.d.\ from a distribution $\mu$ over $X$, and aims to output a binary classifier (a.k.a.\ concept or hypothesis) $h:X\rightarrow \{0,1\}$ that is ``close'' to the target classifier $h^*:X\rightarrow\{0,1\}$. 
Here, the performance of $h$ is measured by the classification error $\Pr_{x\sim\mu}[h(x) \ne h^*(x)]$.
In the realizable setting, the target classifier $h^*$ is assumed to be from a given class $\cH$, whereas in the agnostic setting, there is no assumption on $h^*$ but the performance of $h$ is measured relative to the best classifier in the class $\cH$.
The main result of VC theory states that the sample complexity of PAC learning in both the realizable and the agnostic settings are characterized by the \emph{VC dimension} \citep*{MR3408730} of the class $\cH$, which has a simple combinatorial definition.
The success of VC theory raises the possibility that the sample complexity of other learning tasks may also have natural characterizations. 
Below we discuss two such learning tasks that are most relevant to our work: predictor learning and distribution-specific learning.

The extension of PAC learning to predictor learning dates back to \citep*{MR1292865}, in which predictors were termed probabilistic concepts. 
In predictor learning, binary classifiers are replaced by predictors whose predictions take values in $[0,1]$.
\citet{MR1292865} introduced the notion of \emph{fat-shattering dimension} (see \sectionref{sec:fat} for a precise definition) as a generalization of VC dimension to predictor classes and showed a lower bound on the sample complexity of learning a predictor class by its fat-shattering dimension.
\citet*{MR1481318} and
\citet*{MR1408000} complemented the result with corresponding upper bounds and concluded that a predictor class is learnable from finitely many examples if and only if it has finite fat-shattering dimension. 
Quantitatively,
these papers showed that
the sample complexity of learning a predictor in a class $\cP$ within $\ell_1$ error $\varepsilon$ is characterized by
\[
(1/\varepsilon)^{O(1)}\fat_{\cP}\big(\varepsilon^{\Theta(1)}\big),
\]
assuming we want the success probability to be at least a constant, say $9/10$.
Moreover,
\citet{MR1408000} extended the characterization to the agnostic setting, where the objective is the mean absolute error between the learned prediction and the actual outcome.
Later,
the sample complexity bounds and the related uniform convergence results from \citet{MR1481318} and \citet{MR1408000} were improved by \citet*{bartlett1995more}, \citet*{MR1629694}, and \citet*{MR1824457}.

Another natural extension of PAC learning is distribution-specific learning. In both PAC learning and predictor learning discussed above, %
performance of the learner is evaluated based on the worst-case distribution $\mu$. These settings are referred to as \emph{distribution-free} due to their lack of dependence on a particular input distribution. 
Since the distribution is usually not the worst-case in practice, \emph{distribution-specific} learning focuses on the performance of the learner on a given distribution $\mu$. In this setting, 
the sample complexity of learning a binary classifier in class $\cH$ to achieve a classification error below $\varepsilon$ is characterized by 
\begin{equation}
\label{eq:distribution-specific-learning}
(1/\varepsilon)^{O(1)}\log N_{\mu}(\cH,\Theta(\varepsilon))
\end{equation}
using the \emph{metric entropy}, i.e., the logarithm of the covering number $N_{\mu}$ of $\cH$ w.r.t.\ the classification error (as a metric), which indeed depends on the specific distribution $\mu$ \citep{MR1122796}.

To compare OI with these previous clasification-error/$\ell_1$-error-based notions of learning in terms of sample complexity, we need a characterization of the sample complexity of OI using similar quantities such as the fat-shattering dimension or the metric entropy. While it is tempting to hope that we might directly apply such notions, OI introduces additional subtlety in that we must consider how the expressiveness of the predictor class $\cP$ \emph{and} the class of distinguishers $\cD$ interact to fully understand the sample complexity requirements. This is in contrast to standard settings where characterizing the expressiveness of the concept class via VC dimension or related notions is sufficient. 

We show a simple example where it is indeed important to consider the interplay between $\cP$ and $\cD$ rather than just considering their contributions independently: Partition the set of inputs $X$ into two equal-sized sets, $X_1$ and $X_2$. We consider a class of predictors that are maximally expressive on $X_2$ and constant on $X_1$: Let $\cP = \{p: X \rightarrow [0, 1] | p(x) = 0,\forall x \in X_1\}$. Similarly, we can define a distinguisher class that are maximally inquisitive on one set and ignores the other set but depending on which set is ignored we will get dramatically different complexity: Define two potential distinguisher classes: $\cD_1 = \{d: X \times \{0, 1\}\rightarrow\{\acc, \rej\} | d(x, b) = \rej, \forall x \in X_1, b \in \{0, 1\}\}$, $\cD_2 = \{d: X \times \{0, 1\}\rightarrow\{\acc, \rej\} | d(x, b) = \rej, \forall x \in X_2, b \in \{0, 1\}\}$. 
$\cD_1$ and $\cD_2$ are symmetric and thus identical in any independent measure of complexity. 
Nevertheless, it is easy to see that while achieving $\varepsilon$-OI on $\cP$ with respect to $\cD_1$ is equivalent to learning a predictor from $\cP$ with $\varepsilon$ $\ell_1$ error on the set $X_2$, learning an $\varepsilon$-OI predictor from $\cP$ with respect to $\cD_2$ is trivial as $\cP$ is constant on all individuals that $\cD_2$ can distinguish between, and therefore any $p \in \cP$ will satisfy perfect OI with respect to $\cD_2$. This example demonstrates the need for a more subtle variation of existing complexity notions in order to tightly characterize the sample complexity of OI.

Prior to our work, \citet{dwork2021outcome} showed that $O(\varepsilon^{-4}\log(|\cD|/\varepsilon\delta))$ examples are sufficient for achieving an advantage below $\varepsilon$ over a distinguisher class $\cD$ with probability at least $1-\delta$ in the distribution-free setting. In this work, we refine the $\log|\cD|$ dependence on $\cD$ to its fat-shattering dimension with a matching lower bound (\sectionref{sec:distribution-free}). In addition, we characterize the sample complexity of OI in the distribution-specific setting with greater generality for every given predictor class $\cP$, placing OI in the same setting as the classification-error/$\ell_1$-error-based notions of learning with improved sample complexity due to a well-tuned objective based on the distinguisher class $\cD$. The interplay between the distinguisher class $\cD$ and the predictor class $\cP$ 
connects our sample complexity characterizations to the intriguing metric entropy duality conjecture in convex geometry, which we discuss more in \sectionref{sec:intro-duality} below.

\subsection{Covering for Distinguishers and Metric Entropy Duality}
\label{sec:intro-duality}
Before we describe our sample complexity characterizations for OI, we would like to discuss the ideas behind the metric-entropy-based sample complexity characterization \eqref{eq:distribution-specific-learning} for distribution-specific learning by \citet{MR1122796}. Our characterizations for distribution-specific OI are based on similar ideas, 
but in a more subtle and arguably surprising way,
making an intriguing connection to the \emph{metric entropy duality conjecture}, a long-standing conjecture in convex geometry.

The idea behind the sample complexity upper bound in \eqref{eq:distribution-specific-learning} is to reduce the size of the classifier class $\cH$ by taking an $\varepsilon$-covering of it.
Consider the space of all binary classifiers endowed with the ``classification error metric'' $\eta(h_1,h_2) = \Pr_{x\sim\mu}[h_1(x) \ne h_2(x)]$. If two classifiers $h_1$ and $h_2$ are close w.r.t.\ this metric, choosing $h_1$ and $h_2$ would result in roughly equal classification errors w.r.t.\ the target predictor. This gives a natural procedure for simplifying the classifier class and consequently controlling the sample complexity: we can replace $\cH$ by an $\varepsilon$-covering $\cH'$ of it with only minor loss in its expressivity. Here, an $\varepsilon$-covering is a subset $\cH'\subseteq \cH$ such that every $h\in\cH$ can find a close companion $h'\in\cH'$ such that $\eta(h,h') \le \varepsilon$.
As the size of $\cH'$ can be bounded by the covering number $N_{\mu}(\cH,\varepsilon)$, 
we can use a relatively small amount of examples to accurately estimate the classification error of every classifier in $\cH'$. This naturally leads to the empirical risk minimization (ERM) algorithm used in \citep{MR1122796}.

To extend the ERM algorithm to outcome indistinguishability, a natural idea is to define the metric $\eta(p_1,p_2)$ between two predictors $p_1$ and $p_2$ to be the maximum distinguishing advantage for $p_1$ and $p_2$ w.r.t.\ the distinguisher class $\cD$, and compute an $\varepsilon$-covering $\cP'$ of the predictor class $\cP$.
However, this direct extension (\algorithmref{alg:erm}) of the ERM algorithm to OI does not give us the right sample complexity upper bound (we show this formally in \appendixref{sec:failure-erm}, especially in \lemmaref{lm:failure-erm-2}). This failure is partly 
because the acceptance probabilities of the distinguishers in $\cD$ are not all preserved on a small sample. Indeed, learning a predictor $p$ with MDA below $\varepsilon$ w.r.t.\ to the target predictor $p^*$ necessarily requires estimating the acceptance probability of every distinguisher in $\cD$ on $p^*$ within error $\varepsilon$ (the estimates are simply the acceptance probabilities on $p$).

To meet the need of estimating the acceptance probabilities of the distinguishers in $\cD$, we use a new algorithm (\algorithmref{alg:distinguisher-covering}) where we compute a covering of $\cD$ w.r.t.\ a metric defined by the predictor class $\cP$, i.e., we flip the roles of $\cP$ and $\cD$ in the covering.
Using this new algorithm, we get a sample complexity upper bound as a function of the metric entropy $\log N_{\mu,\cP}(\cD,\varepsilon)$ of $\cD$ w.r.t.\ $\cP$, but the sample complexity lower bound we get by extending the arguments from \citep{MR1122796} is still a function of the metric entropy $\log N_{\mu,\cD}(\cP,\varepsilon)$ of $\cP$ w.r.t.\ $\cD$. How do we make the upper and lower bounds match? 

Our idea is to first transform distinguishers into the same inner product space as the predictors, and then interpret $\cD$ and $\cP$ as two abstract vector sets $\cF_1$ and $\cF_2$ in the same inner product space that can exchange roles freely. Specifically, combining our upper and lower bounds, we know that the lower bound based on $\log N_{\mu,\cD}(\cP,\varepsilon)$ never exceeds the upper bound based on $\log N_{\mu,\cP}(\cD,\varepsilon)$. If we set $\cF_1  = \cP$ and $\cF_2 = \cD$, a more precise version of what we get is
\begin{equation}
\label{eq:duality-intro}
\log N_{\mu,\cF_2}(\cF_1,\varepsilon) \le O\Big(\varepsilon^{-2}\Big(1 + \log N_{\mu,\cF_2}(\cF_1,\Omega(\varepsilon))\Big)\Big).
\end{equation}
If $\cF_1$ and $\cF_2$ are two arbitrary abstract sets of vectors, we can inversely set $\cF_1 = \cD$ and $\cF_2 = \cP$ in the inequality above and get
\[
\log N_{\mu,\cP}(\cD,\varepsilon) \le O\Big(\varepsilon^{-2}\Big(1 + \log N_{\mu,\cD}(\cP,\Omega(\varepsilon))\Big)\Big).
\]
This inequality is exactly what we need to flip back the roles of $\cP$ and $\cD$ and make our upper bound match our lower bound. 

The key inequality \eqref{eq:duality-intro} that helps match our upper and lower bounds comes from combining the bounds themselves.
Moreover,
this inequality makes an intriguing connection to the long-standing \emph{metric entropy duality conjecture} in convex geometry, which conjectures that \eqref{eq:duality-intro} holds without the $\varepsilon^{-2}$ factor, but with the additional assumption that $\cF_1$ and $\cF_2$ are convex and symmetric around the origin. Without this additional assumption, we show that the quadratic dependence on $1/\varepsilon$ in \eqref{eq:duality-intro} is nearly tight (\lemmaref{lm:tight-duality}).

The metric entropy duality conjecture (formally stated in \conjectureref{conjecture}) was first proposed by \citet*{MR0361822}. While the conjecture remains open, many weaker versions of it have been proved \citep*[e.g.][]{MR1008716}. Most notably, the conjecture was proved by \citet*{MR2113023} in the special case where one of the two convex sets is an ellipsoid.
This result was further strengthened by \citet*{MR2105957}.
\citet*{MR2296760} proved a weaker form of the conjecture with the constant factors replaced by dimension dependent quantities.

\subsection{Our Contributions}

Outcome indistinguishability was originally proposed as a strong fairness and accuracy criterion, potentially requiring large sample sizes to achieve. In this work, we view OI differently as a meaningful notion when we have insufficient data for $\ell_1$-error-based learning.
For this reason, we focus on no-access OI, the simplest form of OI introduced by \citet{dwork2021outcome}.
For no-access OI, we show that (randomized) distinguishers can be converted to vectors in the same inner product space as the predictors (\sectionref{sec:distinguisher-function-individual}), connecting the MDA objective to the \emph{multiaccuracy} objective used by \citet{kim2019multiaccuracy}. This allows us to understand predictor classes $\cP$ and distinguisher classes $\cD$ using geometric notions such as the dual Minkowski norms $\|\cdot\|_{\mu,\cP}, \|\cdot\|_{\mu,\cD}$ and the covering numbers $N_{\mu,\cP}(\cdot,\cdot), N_{\mu,\cD}(\cdot,\cdot)$ defined by these norms (see \sectionref{sec:cover}).

In the distribution-specific setting, we consider realizable OI where the target predictor lies in an arbitrary given predictor class $\cP$. 
Setting the failure probability bound $\delta$ in \theoremref{thm:characterization} to be a constant, say $1/10$,
for every predictor class $\cP$, every distinguisher class $\cD$, every data distribution $\mu$ and every MDA bound $\varepsilon$, we characterize the sample complexity of distribution-specific realizable OI
both as
\begin{equation}
\label{eq:contributions-1}
(1/\varepsilon)^{O(1)}\log N_{\mu,\cD}(\cP ,\Theta(\varepsilon))
\end{equation}
 and as 
\begin{equation}
\label{eq:contributions-2}
(1/\varepsilon)^{O(1)}\log N_{\mu,\cP}(\cD ,\Theta(\varepsilon)).
\end{equation}

Our sample complexity characterizations \eqref{eq:contributions-1} and \eqref{eq:contributions-2} highlight an intriguing connection between learning theory and the metric entropy duality conjecture (\conjectureref{conjecture}) first proposed by \citet*{MR0361822}, which conjectures that
\[
\log N_{\mu,\cF_1}(\cF_2,\varepsilon) \le O(\log N_{\mu,\cF_2}(\cF_1,\Omega(\varepsilon)))
\]
whenever two function classes $\cF_1$ and $\cF_2$ are convex and symmetric around the origin.
Our sample complexity characterizations imply a variant version of metric entropy duality (\theoremref{thm:our-duality}) where $\cF_1$ and $\cF_2$ are not required to be convex and symmetric, which we show is nearly tight (\lemmaref{lm:tight-duality}).

In the distribution-free setting, we focus on the case where $\cP$ contains all possible predictors, which is the setting considered by \citet{dwork2021outcome}. Setting $\delta = 1/10$ in \theoremref{thm:characterization-free}, we show that the sample complexity of distribution-free OI in this setting is characterized by 
\[
(1/\varepsilon)^{O(1)}\fat_{\cD}(\Theta(\varepsilon)).
\]
This characterization extends to \emph{multicalibration} with some modifications (see \remarkref{remark:multicalibration-upper} and \remarkref{remark:multicalibration-lower}). Our result
refines the $\log |\cD|$ in the sample complexity upper bounds by \citet{dwork2021outcome} (for OI) and by \citet{hebert2018multicalibration} (for multicalibration) to the fat-shattering dimension of $\cD$ with a matching lower bound.

In addition, we show that the sample complexity of OI in the agnostic setting behaves very differently from the realizable setting. 
This is in contrast to many common learning settings where the sample complexities of realizable and agnostic learning usually behave similarly
(a recent independent work by \citet{hopkins2021realizable}
gives a unified explanation for this phenomenon).
In both the distribution-free and the distribution-specific settings, we show that the sample complexity of agnostic OI can increase when we remove some distinguishers from $\cD$, and it can become arbitrarily large even when the realizable sample complexity is bounded by a constant (\sectionref{sec:real-separation}).
This also suggests that OI can have larger sample complexity compared to $\ell_1$-error based learning in the agnostic setting. This is because in the agnostic setting, the performance of the learned predictor is measured relative to the best predictor in the class $\cP$, which can have a much better performance when measured using the selected objective of OI than using the $\ell_1$ error.
On the other hand, when the target predictor $p^*$ belongs to the symmetric convex hull of the predictor class $\cP$, we show that the sample complexity in the distribution-specific agnostic setting has the same characterizations as in the realizable setting (\lemmaref{lm:inside-hull}).

\subsection{Related Work}

The notion of outcome indistinguishability originated from the growing research of algorithmic fairness. Specifically, outcome indistinguishability can be treated as a generalization of multiaccuracy and multicalibration introduced by \citet*{hebert2018multicalibration} and \citet*{kim2019multiaccuracy}, in which the goal is to ensure that the learned predictor is accurate in expectation or calibrated conditioned on every subpopulation in a subpopulation class $\cC$.
Roughly speaking, the subpopulation class $\cC$ in multiaccuracy and multicalibration plays a similar role as the distinguisher class $\cD$ in outcome indistinguishability, and this connection has been discussed more extensively in \citep[Section 4]{dwork2021outcome}. Beyond fairness, multicalibration and OI also provide strong accuracy guarantees \citep*[see e.g.][]{rothblum2021multi, blum2020sleeping, zhao2021calibrating, gopalan2021omnipredictors,kim2022universal,burhanpurkar2021scaffolding}. 
For a general predictor class $\cP$ and a subpopulation class $\cC$,
\citet*{NEURIPS2020_9a96876e} showed sample complexity upper bounds of uniform convergence for multicalibration based on the maximum of suitable complexity measures of $\cC$ and $\cP$. They complemented this result with a lower bound which does not grow with $\cC$ and $\cP$.
In comparison, we focus on the weaker no-access OI setting where the sample complexity can be much smaller, and we provide matching upper and lower bounds in terms of the dependence on $\cD$ and $\cP$.

\subsection{Paper Organization}
The remainder of this paper is structured as follows. In \sectionref{sec:preli}, we formally define OI and related notions. We give lower and upper bounds for the sample complexity of distribution-specific realizable OI in \sectionref{sec:distribution-specific}, and turn these bounds into a characterization in \sectionref{sec:duality} via metric entropy duality. We characterize the sample complexity of distribution-free OI in \sectionref{sec:distribution-free}. Finally, in \sectionref{sec:separation} we show a strong separation between the sample complexities of realizable and agnostic OI in both the distribution-free and distribution-specific settings.

\section{Preliminaries}
\label{sec:preli}
We use $X$ to denote an arbitrary non-empty set of individuals throughout the paper.
For simplicity, whenever we say $\mu$ is a distribution over $X$, we implicitly assume that every subset of $X$ is measurable w.r.t.\ $\mu$ (this holds e.g.\ when $\mu$ is a discrete distribution), although all the results in this paper naturally extend to more general distributions under appropriate measurability assumptions. We use $\Delta_X$ to denote the set of all distributions over $X$ satisfying the implicit assumption.
For two sets $A$ and $B$, we use $B^A$ to denote the class of all functions $f:A\rightarrow B$. 
Given $r\in [0,1]$, we use $\ber(r)$ to denote the Bernoulli distribution over $\{0,1\}$ with mean $r$.
We use $\log$ to denote the base-$2$ logarithm.
\subsection{Outcome Indistinguishability}

\emph{Outcome indistinguishability} is a theoretical framework introduced by \cite{dwork2021outcome} that aims to guarantee that the outcomes produced by some learned predictor $p: X \rightarrow [0, 1]$ are indistinguishable to a predetermined class of distinguishers $\cD$ from outcomes sampled from the true probabilities for each individual defined by $p^*: X \rightarrow [0,1]$.   

The distinguishing task in outcome indistinguishability consists of drawing an individual $x \in X$ according to some population distribution $\mu\in\Delta_X$ and then presenting the distinguisher with an outcome/individual pair $(x, o)$ where $o$ is either sampled according to the true predictor $p^*$ from the Bernoulli distribution $\ber(p^*(x))$, or sampled according to the learned predictor $p$ from $\ber(p(x))$. 
Taking a pair $(x,o)\in X \times \{0,1\}$, a distinguisher $d$ outputs $d(x,o) = \textnormal{ACCEPT}$ or $d(x,o) = \textnormal{REJECT}$. 
We allow distinguishers to be randomized.

Given a predictor $p$ and a distribution $\mu$ over $X$, we define $\mu_p$ to be the distribution of pairs $(x, o) \in X \times \{0, 1\}$ drawn using the process described above such that $x \sim \mu$ and $o \sim \ber(p(x))$.
With this notation in hand, we can now provide a formal definition of outcome indistinguishability:

\begin{definition}[No-Access Outcome Indistinguishability, \cite{dwork2021outcome}]
\label{def:no-access-OI}
Let $\mu\in\Delta_X$ be a distribution over a set of individuals $X$ and $p^*: X \rightarrow [0, 1]$ be some target predictor for the set of individuals. For a class of distinguishers $\cD$ and $\varepsilon > 0$, a predictor $p: X \rightarrow [0, 1]$ satisfies $(\cD, \varepsilon)$-outcome indistinguishability (OI) if for every $d \in \cD$,
\begin{equation}
\label{eq:no-access-OI}
\left| \Pr_{(x, o) \sim \mu_{p}}\left[d(x, o) = \acc \right] - \Pr_{(x, o^*) \sim \mu_{p^*}}\left[d(x, o^*) = \acc \right]\right| \leq \varepsilon.
\end{equation}
\end{definition}

We refer to the left-hand-side of \eqref{eq:no-access-OI} as the \emph{distinguishing advantage} (or simply \emph{advantage}) of the distinguisher $d$, denoted $\adv_{\mu,d}(p, p^*)$. Given a class $\cD$ of distinguishers, we use $\adv_{\mu,\cD}(p,p^*)$ to denote the supremum $\sup_{d\in\cD}\adv_{\mu,d}(p, p^*)$. According to \definitionref{def:no-access-OI}, a learned predictor $p$ satisfies $(\cD,\varepsilon)$-OI if and only if $\adv_{\mu,\cD}(p,p^*) \le \varepsilon$.

There are many different extensions of OI to settings in which the power of distinguishers is extended to allow access to additional information about the learned predictor $p$ or access to more than one individual/outcome pair. We will be focusing on this most basic case in which the distinguisher only has access to a single pair $(x,o)$ and has no additional access to the learned predictor $p$ (hence the name \emph{No-Access OI}). 

\subsubsection{OI Algorithms}
An OI algorithm (or learner) takes examples $(x_1,o_1),\ldots,(x_n,o_n)$ drawn i.i.d.\ from $\mu_{p^*}$ and outputs a predictor $p$ that satisfies $(\cD,\varepsilon)$-OI with probability at least $1-\delta$:

\begin{definition}[No-Access OI Algorithms]
Given a target predictor $p^*\in[0,1]^X$, a class of distinguishers $\cD$, a distribution $\mu\in \Delta_X$, an advantage bound $\varepsilon \ge 0$, a failure probability bound $\delta \ge 0$, and a nonnegative integer $n$, we use $\oi_{n}(p^*,\cD,\varepsilon,\delta,\mu)$ to denote the set of all (possibly randomized and inefficient) algorithms $\cA$ with the following properties:
\begin{enumerate}
\item $\cA$ takes $n$ examples $(x_1,o_1),\ldots,(x_n,o_n)\in X\times\{0,1\}$ as input and outputs a predictor $p\in [0,1]^X$;
\item when the input examples are drawn i.i.d.\ from $\mu_{p^*}$, with probability at least $1-\delta$ the output predictor $p$ satisfies $\adv_{\mu,\cD}(p,p^*) \le \varepsilon$.
\end{enumerate}
\end{definition}

When seeking to learn an outcome indistinguishable predictor, there are two different tasks that we can consider. On one hand, in what we call the \emph{realizable} case, we assume that the target predictor $p^*$ lies in a known predictor class $\cP\subseteq[0,1]^X$, and we seek to achieve low distinguishing advantages over all distinguishers in the class $\cD$.\footnote{
Throughout the paper, we implicitly assume that all predictor classes and distinguisher classes are non-empty.
} 
Alternatively, in the \emph{agnostic} case, we can imagine a situation in which nothing is known about the target predictor $p^*$, but the performance of the learned predictor is measured relative to the best predictor in $\cP$. In both the agnostic and realizable settings, we can also specify whether we measure performance on the worst-case distribution $\mu$ over individuals in $X$, or on some particular distribution $\mu$ given to the learner. We call these the \emph{distribution-free} and \emph{distribution-specific} settings, respectively. 

\begin{definition}[Algorithms in Various Settings]
Given a predictor class $\cP$, a class of distinguishers $\cD$, a distribution $\mu$, an advantage bound $\varepsilon$, a failure probability bound $\delta$, and a nonnegative integer $n$, we define the sets of algorithms that solve various OI tasks using $n$ examples as follows:
\begin{align*}
& \dsr_{n}(\cP,\cD,\varepsilon,\delta,\mu)  = {\bigcap}_{p^*\in\cP}\oi_n(p^*,\cD,\varepsilon,\delta,\mu), \tag{distribution-specific realizable OI}\\
& \dsa_{n}(\cP,\cD,\varepsilon,\delta,\mu)  = {\bigcap}_{p^*\in[0,1]^X}\oi_n(p^*,\cD,\inf_{p\in\cP}\adv_{\mu,\cD}(p,p^*) + \varepsilon,\delta,\mu), \tag{distribution-specific agnostic OI}\\
& \dfr_{n}(\cP,\cD,\varepsilon,\delta)  = {\bigcap}_{\mu'\in\Delta_X}\dsr_n(\cP,\cD,\varepsilon,\delta,\mu'), \tag{distribution-free realizable OI}\\
& \dfa_{n}(\cP,\cD,\varepsilon,\delta)  = {\bigcap}_{\mu'\in\Delta_X}\dsa_n(\cP,\cD,\varepsilon, \delta,\mu'). \tag{distribution-free agnostic OI}
\end{align*}
\end{definition}

We note that while these learning goals are all defined with respect to some predictor class $\cP$, this class simply constrains the possible values of the target predictor $p^*$ (or in the agnostic case, constrains the predictors used to measure the performance of the returned predictor). In particular we do not require any of the OI algorithms to be proper, i.e.\ always output some $p \in \cP$, despite the fact that some of the algorithms discussed in our proofs happen to satisfy this property.

\begin{definition}[Sample complexity]
Given a predictor class $\cP$, a class of distinguishers $\cD$, a distribution $\mu$, an advantage bound $\varepsilon$, a failure probability bound $\delta$,
we define the sample complexity of various OI tasks as follows:
\begin{align*}
& \sampdsr(\cP,\cD,\varepsilon,\delta,\mu) = \inf\{n\in\bZ_{\ge 0}: \dsr_n(\cP,\cD,\varepsilon,\delta,\mu) \ne \emptyset\},\\
& \sampdsa(\cP,\cD,\varepsilon,\delta,\mu) = \inf\{n\in\bZ_{\ge 0}: \dsa_n(\cP,\cD,\varepsilon,\delta,\mu) \ne \emptyset\},\\
& \sampdfr(\cP,\cD,\varepsilon,\delta) = \inf\{n\in\bZ_{\ge 0}: \dfr_n(\cP,\cD,\varepsilon,\delta) \ne \emptyset\},\\
& \sampdfa(\cP,\cD,\varepsilon,\delta) = \inf\{n\in\bZ_{\ge 0}: \dfa_n(\cP,\cD,\varepsilon,\delta) \ne \emptyset\}.
\end{align*}
\end{definition}
It is clear from the definition that the following monotonicity properties hold: for $\cP' \subseteq \cP,\cD' \subseteq \cD, \varepsilon' \ge \varepsilon, \delta'\ge \delta$,
\begin{align}
& \sampdsr(\cP',\cD',\varepsilon',\delta',\mu) \le \sampdsr(\cP,\cD,\varepsilon,\delta,\mu),\notag \\
& \sampdsa(\cP',\cD,\varepsilon',\delta',\mu) \le \sampdsa(\cP,\cD,\varepsilon,\delta,\mu),\notag \\
& \sampdfr(\cP',\cD',\varepsilon',\delta') \le \sampdfr(\cP,\cD,\varepsilon,\delta), \notag\\
& \sampdfa(\cP',\cD,\varepsilon',\delta') \le \sampdfa(\cP,\cD,\varepsilon,\delta). \label{eq:samp-monotone-1}
\end{align}
Note that the sample complexities in the agnostic setting are not guaranteed to be monotone w.r.t.\ $\cD$ (see \sectionref{sec:real-separation}). It is also clear from definition that
\begin{align}
& \sampdsr(\cP,\cD,\varepsilon,\delta,\mu) \le \sampdfr(\cP,\cD,\varepsilon,\delta),\notag\\
& \sampdsa(\cP,\cD,\varepsilon,\delta,\mu) \le \sampdfa(\cP,\cD,\varepsilon,\delta). \label{eq:samp-monotone-2}
\end{align}

\subsubsection{No-Access Distinguishers as Functions of Individuals}
\label{sec:distinguisher-function-individual}

In the standard definition of OI, distinguishers are thought of as randomized algorithms that take as input an individual-outcome pair $(x,o)\in X\times \{0,1\}$ and output \acc\ or \rej . However, there is a natural way to transform every no-access distinguisher $d$ into a function $f_d$ that maps every individual $x \in X$ to a label $y \in [-1, 1]$ in such a way that the distinguishing advantage of $d$ can be recovered from $f_d$. 

In particular, given a randomized distinguisher $d$, we define $f_d: X \rightarrow [-1, 1]$ such that 
\[
f_d(x) = \Pr[d(x, 1) = \acc] - \Pr[d(x, 0) = \acc] \quad \textnormal{for all}\ x\in X.
\]
Given the function $f_d$, we show that we can always recover the distinguishing advantage of the original distinguisher $d$:

\begin{lemma}
\label{lm:distinguisher-transform}
For any two predictors $p_1, p_2: X \rightarrow [0,1]$,
$$\adv_{\mu,d}(p_1,p_2) = |\bE_{x\sim \mu}[f_d(x)(p_1(x) - p_2(x))]|.$$
\end{lemma}

\begin{proof}
For any $x \in X$ and any two possible outcomes $(o_1,o_2)\in \{(0, 0), (0,1), (1, 0), (1, 1)\}$, it is easily verifiable that 
$$\Pr[d(x, o_1) = \acc] - \Pr[d(x, o_2) = \acc] = f_d(x)(o_1 - o_2),$$
where the probabilities are over the internal randomness of the distinguisher $d$.
The lemma is proved by the following chain of equations:
\begin{align*}
&\adv_{\mu,d}(p_1,p_2) \\
= {} & \left|\Pr_{(x, o) \sim \mu_{p_1}}\left[d(x, o) = \acc\right] - \Pr_{(x, o) \sim \mu_{p_2}}\left[d(x, o) = \acc\right]\right|\\
= {} & \left|\bE\limits_{x\sim\mu, o_1\sim \ber(p_1(x)), o_2\sim\ber(p_2(x))}[\Pr[d(x, o_1) = \acc] - \Pr[d(x, o_2) = \acc]]\right|\\
= {} & \left|\bE\limits_{x\sim\mu, o_1\sim \ber(p_1(x)), o_2\sim\ber(p_2(x))}[f_d(x)(o_1 - o_2)]\right|\\
= {} & \left|\bE\limits_{x\sim\mu}[f_d(x)(p_1(x) - p_2(x))]\right|.
\ifdefined\preprint
\qedhere
\fi
\end{align*}
\end{proof}
Note that the transformation from a distinguisher $d$ to the function $f_d\in[-1,1]^X$ is onto: given any function $f\in[-1,1]^X$, we can construct a distinguisher $d$ such that $f_d = f$ in the following way: if $f(x) \ge 0$, distinguisher $d$ accepts $(x,1)$ with probability $f(x)$, and accepts $(x,0)$ with probability $0$; if $f(x) < 0$, distinguisher $d$ accepts $(x,0)$ with probability $-f(x)$, and accepts $(x,1)$ with probability $0$.

\lemmaref{lm:distinguisher-transform} shows that all distinguishers $d$ mapped to the same function $f_d$ have equal distinguishing advantages.
It also shows that no-access OI has the same error objective as \emph{multiaccuracy} considered by \citet{kim2019multiaccuracy}.
From now on, 
when we refer to a distinguisher $d$, it should be interpreted as the function $f_d\in [-1,1]^X$. Similarly, 
we think of a distinguisher class $\cD$ as a non-empty subset of $[-1,1]^X$.

\subsection{Inner Product, Norm, and Covering Number}
\label{sec:cover}
The set $\bR^X$ of all real-valued functions on $X$ is naturally a linear space: for all $f_1,f_2\in\bR^X$, we define $f_1 + f_2 = g \in\bR^X$ where $g(x) = f_1(x) + f_2(x)$ for all $x\in X$, and for $f\in\bR^X$ and $r\in\bR$, we define $r f = h\in\bR^X$ where $h(x) = r f(x)$ for all $x\in X$.

For any function class $\cF\subseteq \bR^X$ and a real number $r$, we use $r\cF$ to denote $\{rf: f\in\cF\}$ and use $-\cF$ to denote $(-1)\cF$. For any two function classes $\cF_1,\cF_2\subseteq \bR^X$, we define
\begin{align*}
\cF_1 + \cF_2 &= \{f_1 + f_2: f_1\in\cF_1, f_2\in\cF_2\}, \textnormal{and}\\
\cF_1 - \cF_2 &= \{f_1 - f_2: f_1\in\cF_1, f_2\in\cF_2\}.
\end{align*}

For any non-empty function class $\cF\subseteq \bR^X$, we say a function $f\in\bR^X$ is in the convex hull of $\cF$ if there exist $n\in\bZ_{>0}$, $f_1,\ldots,f_n\in\cF$ and $r_1,\ldots,r_n\in\bR_{\ge 0}$ such that $r_1 + \cdots + r_n = 1$ and $f = r_1f_1 + \cdots + r_n f_n$. We use $\bar \cF$ to denote the symmetric convex hull of $\cF$ consisting of all functions in the convex hull of $\cF\cup (-\cF)$. When $\cF$ is empty, we define $\bar\cF$ to be the class consisting of only the all zeros function.

We say a function $f\in\bR^X$ is bounded if there exists $M > 0$ such that $|f(x)| \le M$ for all $x\in X$. We say a function class $\cF\subseteq \bR^X$ is bounded if there exists $M > 0$ such that $|f(x)|\le M$ for all $f\in\cF$ and $x\in X$.

For two bounded functions $f_1,f_2\in \bR^X$, we define their inner product w.r.t.\ distribution $\mu$ as
\[
\langle f_1,f_2 \rangle_\mu = \bE_{x\sim \mu}[f_1(x)f_2(x)].
\]
Although we call the above quantity an inner product for simplicity, it may not be positive definite ($\langle f, f\rangle_\mu = 0$ need not imply $f(x) = 0$ for all $x\in X$).
For a non-empty bounded function class $\cF_1\subseteq \bR^X$ and a bounded function $f\in\bR^X$, we define the dual Minkowski norm of $f$ w.r.t.\ $\cF_1$ to be
\[
\|f\|_{\mu,\cF_1} = \sup_{f_1\in \cF_1}|\langle f, f_1\rangle_{\mu}|.
\]
If $\cF_1$ is empty, we define $\|f\|_{\mu,\cF_1} = 0$.
The norm $\|\cdot\|_{\mu,\cF_1}$ is technically only a semi-norm as it may not be positive definite, but whenever it is positive definite, it is the dual norm of the Minkowski norm induced by (the closure of) $\bar\cF_1$ for finite $X$ \citep[see e.g.][Section 2.1]{MR3210796}.

For a non-empty bounded function class $\cF_2\subseteq \bR^X$ and $\varepsilon \ge 0$, we say a subset $\cF_2'\subseteq \cF_2$ is an $\varepsilon$-covering of $\cF_2$ w.r.t.\ the norm $\|\cdot \|_{\mu,\cF_1}$ if for every $f_2\in\cF_2$, there exists $f_2'\in\cF_2'$ such that $\|f_2 - f_2'\|_{\mu,\cF_1}\le \varepsilon$.
We define the covering number $N_{\mu,\cF_1}(\cF_2,\varepsilon)$ of $\cF_2$ w.r.t.\ the norm $\|\cdot \|_{\mu,\cF_1}$ to be the minimum size of such an $\varepsilon$-covering of $\cF_2'$. We refer to the logarithm of the covering number, $\log N_{\mu,\cF_1}(\cF_2,\varepsilon)$, as the metric entropy.

The following basic facts about the covering number are very useful:
\begin{lemma}
\label{lm:covering-basics}
We have the following facts:
\begin{enumerate}
\item \label{item:basics-monotone} if $\hat \cF_1\subseteq \cF_1\subseteq \bR^X$, then $N_{\mu,\hat \cF_1}(\cF_2,\varepsilon) \le N_{\mu,\cF_1}(\cF_2,\varepsilon)$;
\item \label{item:basics-hull} $N_{\mu,\cF_1}(\cF_2,\varepsilon) = N_{\mu,\bar\cF_1}(\cF_2,\varepsilon)$;
\item \label{item:basics-shift} for every bounded function $f\in\bR^X$, $N_{\mu,\cF_1}(\cF_2 + \{f\},\varepsilon) = N_{\mu,\cF_1}(\cF_2,\varepsilon)$;
\item \label{item:basics-homogeneity} for every $a, b\in\bR_{>0}$, $N_{\mu,a\cF_1}(b\cF_2,ab\varepsilon) = N_{\mu,\cF_1}(\cF_2,\varepsilon)$.
\end{enumerate}
\end{lemma}
\subsection{Distinguishing Advantages as Inner Products and Norms}
The inner products and norms provide convenient ways for describing distinguishing advantages. Given a distinguisher $d\in [-1,1]^X$ and two predictors $p_1,p_2\in[0,1]^X$, \lemmaref{lm:distinguisher-transform} tells us that
\begin{align*}
& \adv_{\mu,d}(p_1,p_2) = |\langle  d, p_1 - p_2 \rangle_\mu|, \quad \textnormal{and thus}\\
& \adv_{\mu,\cD}(p_1,p_2) = \|p_1 - p_2\|_{\mu,\cD} \quad \textnormal{for all distinguisher classes}\ \cD\subseteq[-1,1]^X.
\end{align*}
Using these representations for advantages, we can prove the following lemma relating advantages to the $\ell_1$ error:
\begin{lemma}
\label{lm:advantage-l1}
It holds that $\adv_{\mu,\cD}(p_1,p_2) = \|p_1 - p_2\|_{\mu,\cD} \le \bE_{x\sim\mu}[|p_1(x) - p_2(x)|]$. Moreover, when $\{-1,1\}^X\subseteq \cD$, $\adv_{\mu,\cD}(p_1,p_2) = \bE_{x\sim\mu}[|p_1(x) - p_2(x)|]$.
\end{lemma}
\begin{proof}
To prove the first statement, we recall that
$$\adv_{\mu, \cD} (p_1,p_2) = \|p_1 - p_2\|_{\mu,\cD} = \sup_{d \in \cD} |\bE_{x\sim\mu}[(p_1(x) - p_2(x))d(x)]|.$$
Because $d(x) \in [-1, 1]$ for all $d\in\cD$ and $x\in X$, we are guaranteed that $|(p_1(x) - p_2(x))d(x)| \leq |p_1(x) - p_2(x)|$, which gives
$$\sup_{d \in \cD} |\bE_{x\sim\mu}[(p_1(x) - p_2(x))d(x)]| \le \sup_{d \in \cD} \bE_{x\sim\mu}[|(p_1(x) - p_2(x))d(x)|] \leq \bE_{x\sim\mu}[|p_1(x) - p_2(x)|],$$
as desired. 

For the second statement, consider the distinguisher $d$ defined such that $d(x) = 1$ if $p_1(x) \ge p_2(x)$ and $d(x) = -1$ otherwise. For all $x\in X$, distinguisher $d$ satisfies 
\[
(p_1(x) - p_2(x))d(x) = |p_1(x) - p_2(x)|.
\]
Therefore,
$$\adv_{\mu, d} (p_1,p_2) = |\bE_{x\sim\mu}[(p_1(x) - p_2(x))d(x)]| = \bE_{x\sim\mu}[|p_1(x) - p_2(x)|].$$
Since $d\in \{-1, 1\}^X\subseteq \cD$,
this proves the second statement.
\end{proof}
\subsection{Fat-Shattering Dimension}
\label{sec:fat}
Given a function class $\cF\subseteq \bR^X$ and $\gamma \ge 0$, we say $x_1,\ldots,x_n\in X$ are $\gamma$-fat shattered by $\cF$ w.r.t.\ $r_1,\ldots,r_n\in\bR$ if for every $(b_1,\ldots,b_n)\in\{-1,1\}^n$, there exists $f\in\cF$ such that
\[
b_i(f(x_i) - r_i) \ge \gamma \quad \textnormal{for all}\ i \in\{1,\ldots,n\}.
\]
We sometimes omit the mention of $r_1,\ldots,r_n$ and say $x_1,\ldots,x_n$ is $\gamma$-fat shattered by $\cF$ if such $r_1,\ldots,r_n$ exist.
The $\gamma$-fat-shattering dimension of $\cF$ introduced first by \citet{MR1292865} is defined to be
\[
\fat_\cF(\gamma) = \sup\{n\in\bZ_{\ge 0}: \textnormal{there exist $x_1,\ldots,x_n\in X$ that are $\gamma$-fat shattered by $\cF$}\}.
\]

\section{Sample Complexity of Distribution-Specific Realizable OI}
\label{sec:distribution-specific}
In this section,
we give lower and upper bounds for the sample complexity of distribution-specific realizable OI for every given predictor class $\cP$, distinguisher class $\cD$, and distribution $\mu$ over individuals. Our lower bound is based on the metric entropy of $\cP$ w.r.t.\ the norm $\|\cdot\|_{\mu,\cD}$, whereas in our upper bound the roles of $\cP$ and $\cD$ are flipped. In the next section, we remove this role flip and give a complete characterization of the sample complexity using a version of metric entropy duality implied from combining our lower and upper bounds.

\subsection{Lower Bound}
We prove the following lemma showing that the sample complexity of distribution-specific realizable OI is lower bounded by the metric entropy of the predictor class $\cP$ w.r.t.\ the dual Minkowski norm $\|\cdot\|_{\mu,\cD}$ defined by the distinguisher class $\cD$. This lemma generalizes \citep[Lemma 4.8]{MR1122796}, which considered the special case where every predictor is a binary classifier, and the distinguisher class $\cD$ contains all possible distinguishers ($\cD = [-1,1]^X$).
\begin{lemma}
\label{lm:OI-sample-lower}
For every predictor class $\cP\subseteq[0,1]^X$, every distinguisher class $\cD\subseteq [-1,1]^X$, every distribution $\mu\in\Delta_X$, every advantage bound $\varepsilon > 0$, and every failure probability bound $\delta \in (0,1)$,
the following sample complexity lower bound holds for distribution-specific realizable OI:
\[
\sampdsr(\cP,\cD,\varepsilon,\delta,\mu) \ge \log ( (1 - \delta)N_{\mu,\cD}(\cP,2\varepsilon)).
\]
\end{lemma}
\begin{proof}
Define $M = N_{\mu,\cD}(\cP,2\varepsilon)$. 
Let $\cP'$ be the maximum-size subset of $\cP$ such that
\begin{equation}
\label{eq:packing}
\|p_1 - p_2\|_{\mu,\cD} > 2\varepsilon \quad \textnormal{for all distinct}\ p_1,p_2\in\cP'.
\end{equation}
It is clear that $|\cP'|\ge M$ because otherwise $\cP'$ is not a $2\varepsilon$-covering of $\cP$ and we can add one more predictor into $\cP'$ without violating \eqref{eq:packing}, a contradiction with the maximality of $|\cP'|$.

Let $n$ be a nonnegative integer such that $\dsr_n(\cP,\cD,\varepsilon,\delta,\mu) \ne \emptyset$. Our goal is to prove
\begin{equation}
\label{eq:goal-OI-lower}
n \ge \log ( (1 - \delta)M).
\end{equation}
Let $\cA$ be an algorithm in $\dsr_n(\cP,\cD,\varepsilon,\delta,\mu)$.
We draw a predictor $p^*$ uniformly at random from $\cP'$, and draw examples $(x_1,o_1),\ldots,(x_n,o_n)$ i.i.d.\ from $\mu_{p^*}$.
We say algorithm $\cA$ succeeds if it outputs $p$ such that $\|p - p^*\|_{\mu,\cD}\le \varepsilon$. 
By assumption, when $(x_1,o_1),\ldots,(x_n,o_n)$ are given as input, algorithm $\cA$ succeeds with probability at least $1 - \delta$. Now instead of drawing $x_1,\ldots,x_n$ i.i.d.\ from $\mu$, we fix them so that the success probability is maximized. In other words, we can find fixed $x_1,\ldots,x_n\in X$ such that if we run algorithm $\cA$ on examples $(x_1,o_1),\ldots,(x_n,o_n)$ where $o_i\sim \ber(p^*(x_i))$ and $p^*$ is chosen uniformly at random from $\cP'$, the algorithm succeeds with probability at least $1 - \delta$. Similarly, we can fix the internal randomness of algorithm $\cA$ and assume that $\cA$ is deterministic without decreasing its success probability on $(x_1,o_1),\ldots,(x_n,o_n)$. Now algorithm $\cA$ has at most $2^n$ possible inputs, and thus has at most $2^n$ possible outputs. No output can be the success output for two different choices of $p^*$ from $\cP'$ because of \eqref{eq:packing}. Therefore, the success probability of algorithm $\cA$ is at most $2^n/M$, and thus
\[
2^n / M \ge 1 - \delta.
\]
This implies \eqref{eq:goal-OI-lower}, as desired.
\end{proof}

\subsection{Upper Bound}
We give an algorithm for distribution-specific realizable OI to prove a sample complexity upper bound for it. 
Before we describe our algorithm, 
let us briefly discuss the empirical risk minimization algorithm (\algorithmref{alg:erm}). \citet[Proof of Lemma 4.6]{MR1122796} showed that this natural algorithm works in the special case where 1) every predictor in $\cP$ is a binary classifier, and 2) the distinguisher class $\cD$ contains all possible distinguishers.
When both 1) and 2) are satisfied,
the algorithm gives a sample complexity upper bound of
\begin{equation}
\label{eq:dream-upper}
O((1/\varepsilon)^{O(1)}\log N_{\mu,\cD}(\cP,\varepsilon/2)), 
\end{equation}
which would give a satisfactory sample complexity characterization when combined with our lower bound in \lemmaref{lm:OI-sample-lower}.
However,
in \appendixref{sec:failure-erm}, we show that the algorithm fails when only one of the two conditions 1) and 2) (no matter which) is true.\footnote{
There is a variant of \algorithmref{alg:erm} that minimizes $\loss(p)$ over the entire predictor class $\cP$ instead of the covering $\cP'$. As discussed in \citep{MR1122796}, this variant is not guaranteed to give a sample complexity upper bound close to \eqref{eq:dream-upper} even under both conditions 1) and 2).
Changing the definition of $\loss(p)$ to $\sup_{d\in\cD}|\langle p,d\rangle_\mu - \frac 1n\sum_{i=1}^n d(x_i)o_i|$ (mimicking \algorithmref{alg:distinguisher-covering}) also makes \algorithmref{alg:erm} fail under both conditions 1) and 2). To see this, suppose $\mu$ is the uniform distribution over a finite domain $X$, $\cP = \{p_0,p_1\}$ where $p_0(x) = 0$ and $p_1(x) = 1$ for every $x\in X$,  and $\cD = [-1,1]^X$. Assuming $\varepsilon \in (0,1)$ and $n < |X|/10$, when the target predictor $p^*$ is $p_1$, \algorithmref{alg:erm} always outputs $p_0$ on the new $\loss$ (note that changing the values of $d$ on $x_1,\ldots,x_n$ can significantly change $\frac 1n\sum_{i=1}^n d(x_i)o_i$, but it never changes $\langle p,d \rangle_\mu$ by more than $2n/|X|$).
}
Since neither 1) nor 2) is guaranteed to hold in the distribution-specific realizable OI setting, we 
use a new algorithm (\algorithmref{alg:distinguisher-covering}) to
prove our sample complexity upper bound (\lemmaref{lm:OI-sample-upper}) where the roles of $\cP$ and $\cD$ flip compared to \eqref{eq:dream-upper}.
In \sectionref{sec:duality}, we show how to flip them back to get a sample complexity characterization for distribution-specific realizable OI.

\begin{algorithm2e}
  \SetKwInOut{KwPa}{Parameters}
  \SetKwInOut{KwIn}{Input}
  \SetKwInOut{KwOut}{Output}
  \KwPa{predictor class $\cP$, distinguisher class $\cD$, distribution $\mu$, MDA bound $\varepsilon$, positive integer $n$.}
  \KwIn{examples $(x_1,o_1),\ldots,(x_n,o_n)$.}
  \KwOut{predictor $p\in\cP$. }
  $\cP'\gets$ minimum-size $\varepsilon/2$-covering of $\cP$ w.r.t.\ norm $\|\cdot \|_{\mu,\cD}$\;
  \Return $p\in\cP'$ that minimizes the empirical error
  \[
  \loss(p):=\sup_{d\in\cD}\left|\frac 1n\sum_{i=1}^n d(x_i)(p(x_i) - o_i)\right|;
  \]
  \caption{Empirical Risk Minimization}
  \label{alg:erm}
\end{algorithm2e}

\begin{algorithm2e}
  \SetKwInOut{KwPa}{Parameters}
  \SetKwInOut{KwIn}{Input}
  \SetKwInOut{KwOut}{Output}
  \KwPa{predictor class $\cP$, distinguisher class $\cD$, distribution $\mu$, MDA bound $\varepsilon$, positive integer $n$.}
  \KwIn{examples $(x_1,o_1),\ldots,(x_n,o_n)$.}
  \KwOut{predictor $p\in\cP$. }
  $\cQ\gets \cP - \cP$ \tcc*{Recall that $\cP - \cP = \{p_1 - p_2:p_1,p_2\in\cP\}$.}
  $\cD'\gets$ minimum-size $\varepsilon/2$-covering of $\cD$ w.r.t.\ norm $\|\cdot \|_{\mu,\cQ}$\;
  \Return $p\in\cP$ that $\varepsilon/16$-minimizes
  \[
  \loss(p):=\sup_{d\in\cD'}\left|\langle p, d\rangle_\mu - \frac 1n\sum_{i=1}^nd(x_i)o_i\right|,
  \]
  i.e., $\loss(p) \le \inf_{p'\in\cP}\loss(p') + \varepsilon/16$\;
  \caption{Distinguisher Covering}
  \label{alg:distinguisher-covering}
\end{algorithm2e}

Our sample complexity upper bound is based on the following analysis of \algorithmref{alg:distinguisher-covering}:
\begin{lemma}
\label{lm:alg-distinguisher-covering}
For every predictor class $\cP\subseteq[0,1]^X$, every distinguisher class $\cD\subseteq[-1,1]^X$, every distribution $\mu\in\Delta_X$, every advantage bound $\varepsilon\in(0,1)$, and every failure probability bound $\delta\in(0,1)$, there exists a positive integer
\begin{align}
n & \le O\big(\varepsilon^{-2}\big(\log N_{\mu,\cQ}(\cD,\varepsilon/2) + \log (2/\delta)\big)\big) \label{eq:OI-sample-upper-1}\\
& \le O\big(\varepsilon^{-2}\big(\log N_{\mu,\cP}(\cD,\varepsilon/4) + \log(2/\delta)\big)\big) \label{eq:OI-sample-upper-2}
\end{align}
such that \algorithmref{alg:distinguisher-covering} belongs to
\[
{\bigcap}_{p^*\in[0,1]^X}\oi_n(p^*,\cD,3 \inf_{p\in\cP}\adv_{\mu,\cD}(p,p^*) + \varepsilon,\delta,\mu),
\]
where
$\cQ = \cP - \cP = \{p_1 - p_2:p_1,p_2\in\cP\}$.
\end{lemma}

\begin{proof}
We first note that $\|f\|_{\mu,\cQ} \le 2\|f\|_{\mu,\cP}$ for all bounded functions $f\in \bR^X$, which implies that
\[
N_{\mu,\cQ}(\cD,\varepsilon/2) \le N_{\mu,\cP}(\cD,\varepsilon/4).
\]
This proves inequality \eqref{eq:OI-sample-upper-2}.

Define $M = N_{\mu,\cQ}(\cD,\varepsilon/2)$ and $\varepsilon_0 = \inf_{p\in\cP}\adv_{\mu,\cD}(p,p^*)$ for an arbitrary $p^*\in[0,1]^X$.
It remains to prove that \algorithmref{alg:distinguisher-covering} belongs to
\[
{\bigcap}_{p^*\in[0,1]^X}\oi_n(p^*,\cD,3 \varepsilon_0 + \varepsilon,\delta,\mu)
\]
for some $n = O(\varepsilon^{-2}(\log M + \log (2/\delta) ))$ determined below. Given $n$ examples $(x_1,o_1),\ldots,(x_n,o_n)$, we define $K(d) = \frac 1n\sum_{i=1}^no_id(x_i)$ for every distinguisher $d\in\cD'$, where $\cD'$ is the minimum-size $\varepsilon/2$ covering of $\cP$ computed in \algorithmref{alg:distinguisher-covering}.
By definition, $|\cD'| = M$, so by the Chernoff bound and the union bound, for some $n = O(\varepsilon^{-2}(\log M + \log (2/\delta) ))$, with probability at least $1-\delta$,
\begin{equation}
\label{eq:sample-upper-concentration}
|K(d) - \langle p^*, d\rangle_\mu|\le \varepsilon / 16 \quad \textnormal{for all}\ d\in\cD'.
\end{equation}
Assuming that \eqref{eq:sample-upper-concentration} is true, it suffices to prove that the output $p$ of \algorithmref{alg:distinguisher-covering} satisfies $\|p - p^*\|_{\mu,\cD} \le 3\varepsilon_0 + \varepsilon$.
Let $\tilde p\in\cP$ be a predictor that satisfies $\|\tilde p - p^*\|_{\mu,\cD} \le \varepsilon_0 + \varepsilon/16$.
We have 
\[
\sup_{d\in\cD'}|K(d) - \langle \tilde p, d\rangle_\mu| \le \sup_{d\in\cD'}|\langle p^*,d\rangle_\mu - \langle \tilde p, d\rangle_\mu| + \varepsilon/16 \le \varepsilon_0 + \varepsilon/8.
\]
The output predictor $p$ of \algorithmref{alg:distinguisher-covering} satisfies
\[
\sup_{d\in\cD'}|K(d) - \langle p, d\rangle_\mu| \le \sup_{d\in\cD'}|K(d) - \langle \tilde p, d\rangle_\mu| + \varepsilon/16 \le \varepsilon_0 + 3\varepsilon/16.
\]
Therefore,
\begin{align*}
\sup_{d\in\cD'}|\langle p - \tilde p, d\rangle_\mu| & \le \sup_{d\in \cD'}|\langle p, d\rangle_\mu - K(d)| + \sup_{d\in\cD'}|\langle \tilde p, d\rangle_\mu - K(d)|\\
& \le (\varepsilon_0 + \varepsilon/8) + (\varepsilon_0 + 3\varepsilon/16)\\
& \le 2\varepsilon_0 + 5\varepsilon / 16.
\end{align*}
Since $p - \tilde p\in \cQ$ and $\cD'$ is an $\varepsilon/2$-covering w.r.t.\ $\|\cdot\|_{\mu,\cQ}$,
\[
\|p - \tilde p\|_{\mu,\cD} = 
\sup_{d\in\cD}|\langle p - \tilde p, d\rangle_\mu| \le \sup_{d\in\cD'}|\langle p - \tilde p, d\rangle_\mu| + \varepsilon / 2 \le 2\varepsilon_0 + 13\varepsilon/16.
\]
Finally,
\[
\|p - p^*\|_{\mu,\cD} \le \|p - \tilde p\|_{\mu,\cD} + \|\tilde p - p^*\|_{\mu,\cD} \le (2\varepsilon_0 + 13\varepsilon/16) + (\varepsilon_0 + \varepsilon/16) \le 3\varepsilon_0 + \varepsilon,
\]
as desired.
\end{proof}
We are now ready to state and prove our sample complexity upper bound for distribution-specific realizable OI.
\begin{lemma}
\label{lm:OI-sample-upper}
For every predictor class $\cP\subseteq[0,1]^X$, every distinguisher class $\cD\subseteq[-1,1]^X$, every distribution $\mu\in\Delta_X$, every advantage bound $\varepsilon > 0$, and every failure probability bound $\delta\in(0,1)$, 
the following sample complexity upper bound holds for distribution-specific realizable OI:
\begin{align*}
\sampdsr(\cP,\cD,\varepsilon,\delta,\mu) & \le O\big(\varepsilon^{-2}\big(\log N_{\mu,\cQ}(\cD,\varepsilon/2) + \log (2/\delta)\big)\big) \\
& \le O\big(\varepsilon^{-2}\big(\log N_{\mu,\cP}(\cD,\varepsilon/4) + \log(2/\delta)\big)\big),
\end{align*}
where $\cQ = \cP - \cP = \{p_1 - p_2:p_1,p_2\in\cP\}$.
\end{lemma}

\begin{proof}
When $\varepsilon \ge 1$, we have $\sampdsr(\cP,\cD,\varepsilon,\delta,\mu) = 0$ and the lemma is trivially true. We assume $\varepsilon\in(0,1)$ henceforth.

For every $p^*\in\cP$, we have $\inf_{p\in\cP}\adv_{\mu,\cD}(p,p^*) = 0$, so by \lemmaref{lm:alg-distinguisher-covering}, there exists a positive integer $n$ satisfying \eqref{eq:OI-sample-upper-1} and \eqref{eq:OI-sample-upper-2} such that
\begin{align*}
\emptyset\ne
{\bigcap}_{p^*\in[0,1]^X}\oi_n(p^*,\cD,3 \inf_{p\in\cP}\adv_{\mu,\cD}(p,p^*) + \varepsilon,\delta,\mu)
& \subseteq 
{\bigcap}_{p^*\in \cP}\oi_n(p^*,\cD,\varepsilon,\delta,\mu)\\
& = \dsr_n(\cP,\cD,\varepsilon,\delta,\mu).
\end{align*}
This completes the proof by the definition of $\sampdsr(\cP,\cD,\varepsilon,\delta,\mu)$.
\end{proof}

\section{Metric Entropy Duality}
\label{sec:duality}
In the previous section, we proved lower and upper bounds for the sample complexity of distribution-specific realizable OI,
but these bounds do not yet give a satisfactory sample complexity characterization because of the exchanging roles of the predictor class $\cP$ and the distinguisher class $\cD$ in the lower and upper bounds. 
In this section, we solve the issue by proving the following theorem:

\begin{theorem}
\label{thm:our-duality}
There exists an absolute constant $c \ge 1$ with the following property. For $M_1,M_2 > 0$, let $\cF_1\subseteq[-M_1,M_1]^X$ and $\cF_2\subseteq[-M_2,M_2]^X$ be two non-empty bounded function classes. For any distribution $\mu\in\Delta_X$ and any $\varepsilon > 0$, it holds that
\[
\log N_{\mu,\cF_2}(\cF_1,\varepsilon) \le c(M_1M_2/\varepsilon)^2(1 + \log N_{\mu,\cF_1}(\cF_2,\varepsilon/8)).
\]
\end{theorem}
Before we prove this theorem, we note that it has a similar statement to the long-standing metric entropy duality conjecture proposed first by \citet{MR0361822}. The conjecture can be stated as follows using our notations \citep[for other equivalent statements of the conjecture, see e.g.][]{MR2105957}:
\begin{conjecture}[\citet{MR0361822}]
\label{conjecture}
There exist absolute constants $c_1,c_2 \ge 1$ with the following property.
For any two bounded function classes $\cF_1,\cF_2\subseteq \bR^X$ over a non-empty finite set $X$, if $\cF_1$ and $\cF_2$ are convex and symmetric (i.e.\ $\bar\cF_1 = \cF_1,\bar\cF_2 = \cF_2$), then for any distribution $\mu\in\Delta_X$ and any $\varepsilon >0$, 
\[
\log N_{\mu,\cF_2}(\cF_1,\varepsilon) \le c_1 \log N_{\mu,\cF_1}(\cF_2,\varepsilon/c_2).
\]
\end{conjecture}

Compared to \conjectureref{conjecture},
our \theoremref{thm:our-duality} gives a variant of metric entropy duality which does not require $\cF_1$ and $\cF_2$ to be convex and symmetric, but has the constant $c_1$ in \conjectureref{conjecture} replaced by a quantity dependent on the granularity $\varepsilon$ and the scale of the functions in $\cF_1$ and $\cF_2$.
Since the predictor class $\cP$ and the distinguisher class $\cD$ are not in general convex and symmetric, 
our \theoremref{thm:our-duality} is more convenient for proving sample complexity characterizations for OI (see \theoremref{thm:characterization}).
In \lemmaref{lm:tight-duality}, we show that the quadratic dependence on $M_1M_2/\varepsilon$ in \theoremref{thm:our-duality} is nearly tight.

Below we prove \theoremref{thm:our-duality} by combining our lower and upper bounds in the previous section.
\begin{proof}
By \lemmaref{lm:covering-basics} Item \ref{item:basics-homogeneity}, we can assume w.l.o.g.\ that $M_1 = M_2 = 1$.
Define $\cD = \cF_2\subseteq[-1,1]^X$, $\cP = \{(1 + f)/2:f\in \cF_1\}\subseteq[0,1]^X$, and $\cQ = \cP - \cP$. 
Combining \lemmaref{lm:OI-sample-lower} and \lemmaref{lm:OI-sample-upper} with $\delta = 1/3$ and $\varepsilon$ replaced by $\varepsilon/4$, there exists a constant $c \ge 1$ such that
\[
\log N_{\mu,\cD}(\cP,\varepsilon/2) \le c\varepsilon^{-2}(1 + \log N_{\mu,\cQ}(\cD,\varepsilon/8)).
\]

Using \lemmaref{lm:covering-basics} Items \ref{item:basics-shift} and \ref{item:basics-homogeneity},
\begin{align}
\log N_{\mu,\cF_2}(\cF_1,\varepsilon) & = \log N_{\mu,\cD}(\cP,\varepsilon/2)\notag\\
& \le c\varepsilon^{-2}(1 + \log N_{\mu,\cQ}(\cD,\varepsilon/8))\notag\\
& \le c\varepsilon^{-2}(1 + \log N_{\mu,\cF_1}(\cD,\varepsilon/8)) \label{eq:switch-2}\\
& = c\varepsilon^{-2}(1 + \log N_{\mu,\cF_1}(\cF_2,\varepsilon/8)),\notag
\end{align}
as desired.
Here, inequality \eqref{eq:switch-2} holds because $\|f\|_{\mu,\cQ}\le \|f\|_{\mu, \cF_1}$ for every bounded function $f\in \bR^X$.
\end{proof}
We are now ready to state and prove our sample complexity characterizations for distribution-specific realizable OI.
\begin{theorem}
\label{thm:characterization}
For every predictor class $\cP\subseteq[0,1]^X$, every distinguisher class $\cD\subseteq[-1,1]^X$, every distribution $\mu\in\Delta_X$, every advantage bound $\varepsilon > 0$, and every failure probability bound $\delta\in(0,1)$, 
the following sample complexity characterizations hold for distribution-specific realizable OI:
\begin{align*}
& \log N_{\mu,\cD}(\cP,2\varepsilon) + \log(1 - \delta) \\
\le {} &
\sampdsr(\cP,\cD,\varepsilon,\delta,\mu)\\
\le {} & O(\varepsilon^{-4}\log N_{\mu,\cD}(\cP,\varepsilon/32) + \varepsilon^{-2}\log(2/\delta)),
\end{align*}
and
\begin{align*}
& \Omega(\varepsilon^2\log N_{\mu,\cP}(\cD,16\varepsilon)) -1 + \log(1 - \delta)\\
\le {} &
\sampdsr(\cP,\cD,\varepsilon,\delta,\mu)\\
\le {} &
O(\varepsilon^{-2}\log N_{\mu,\cP}(\cD,\varepsilon/4) + \varepsilon^{-2}\log(2/\delta)).
\end{align*}
\end{theorem}
Since $\log N_{\mu,\cD}(\cP,\varepsilon)$ is monotone w.r.t.\ $\cD$ (\lemmaref{lm:covering-basics} Item~\ref{item:basics-monotone}), compared to $\ell_1$-error-based learning which corresponds to $\cD = [-1,1]^X$ (see \lemmaref{lm:advantage-l1}), \theoremref{thm:characterization} shows that a selected class $\cD$ of distinguishers helps reduce the sample complexity, or allows us to achieve smaller $\varepsilon$ (and potentially better performance guarantees) with the same sample size. We prove \theoremref{thm:characterization} below.

\begin{proof}
We start with the following chain of inequalities:
\begin{align}
& \Omega(\varepsilon^2\log N_{\mu,\cP}(\cD,16\varepsilon)) -1 + \log(1-\delta)  \notag\\
\le {} & \log N_{\mu,\cD}(\cP,2\varepsilon) + \log(1 - \delta) \tag{by \theoremref{thm:our-duality}}\\
\le {} & \sampdsr(\cP,\cD,\varepsilon,\delta,\mu) \tag{by \lemmaref{lm:OI-sample-lower}}\\
\le {} & O(\varepsilon^{-2}\log N_{\mu,\cP}(\cD,\varepsilon/4) + \varepsilon^{-2}\log(2/\delta)),\label{eq:characterization-2}
\end{align}
where the last inequality is by \lemmaref{lm:OI-sample-upper}.
It remains to show
\begin{equation}
\label{eq:characterization-1}
\sampdsr(\cP,\cD,\varepsilon,\delta,\mu) \le O(\varepsilon^{-4}\log N_{\mu,\cD}(\cP,\varepsilon/32) + \varepsilon^{-2}\log(2/\delta)).
\end{equation}
If $N_{\mu,\cD}(\cP,\varepsilon/32) \le 1$, it is clear that $\sampdsr(\cP,\cD,\varepsilon,\delta,\mu) = 0$, so the inequality is trivial. If $N_{\mu,\cD}(\cP,\varepsilon/32) \ge 2$, 
we have the following inequality by \theoremref{thm:our-duality}:
\[
\varepsilon^{-2}\log N_{\mu,\cP}(\cD,\varepsilon/4)) \le O(\varepsilon^{-4}(1 + \log N_{\mu,\cD}(\cP,\varepsilon/32))) \le O(\varepsilon^{-4}\log N_{\mu,\cD}(\cP,\varepsilon/32)).
\]
Inequality \eqref{eq:characterization-1} is proved by plugging the inequality above into \eqref{eq:characterization-2}.
\end{proof}

We complete the section by showing near-tightness of \theoremref{thm:our-duality}.

\begin{lemma}
\label{lm:tight-duality}
There exists a constant $c > 0$ such that for all $M_1>0,M_2>0,\varepsilon\in(0,M_1M_2/2)$, there exist a ground set $X$, a distribution $\mu$ over $X$, and function classes $\cF_1\subseteq[-M_1,M_2]^X,\cF_2\subseteq[-M_2,M_2]^X$ such that
\[
\log N_{\mu,\cF_2}(\cF_1,\varepsilon) \ge c(1 + \log |\cF_2|)\left(\frac{M_1M_2}{\varepsilon}\right)^2\left(\log\left(\frac {M_1M_2}{\varepsilon}\right)\right)^{-1}.
\]
\end{lemma}
\begin{proof}
By \lemmaref{lm:covering-basics} Item \ref{item:basics-homogeneity}, we can assume w.l.o.g.\ that $M_1 = M_2 = 1$.
Now we have $\varepsilon < M_1M_2/2 = 1/2$, and $1/2\varepsilon^2 > 2$.
Let $m$ be the largest integer power of $2$ such that $m \le 1/2\varepsilon^2$. We choose $X$ to be $\{1,\ldots,m\}$, and choose $\mu$ to be the uniform distribution over $X$. 
Let $\vect$ be the bijection from $\bR^X$ to $\bR^m$ such that
$\vect(f) = (f(1),\ldots,f(m)) \in \bR^m$ for all $f\in\bR^X$.
Define $H_m$ to be the set of vectors in $\{-1,1\}^m$ consisting of the $m$ columns of the Hadamard matrix of size $m\times m$.
We choose $\cF_1 = [-1,1]^X$, and $\cF_2 = \{\vect^{-1}(v):v\in H_m\}$. 
The intuition behind the choice of $\cF_2$ is to keep $|\cF_2|$ small while making $N_{\mu,\cF_2}(\cF_1,\varepsilon)$ large, which by \lemmaref{lm:covering-basics} Items~\ref{item:basics-monotone} and \ref{item:basics-hull} roughly corresponds to making the symmetric convex hull $\bar\cF_2$ large. That is why we use the Hadamard matrix to ensure that the functions $f$ in $\cF_2$ achieve the maximum norm (with $f(x) = \pm 1$ for every $x\in X$) and are orthogonal to each other.

Define $B = \{f\in\bR^X:\|f\|_{\mu,\cF_2} \le \varepsilon\}$.
By the properties of the Hadamard matrix,
the functions in $\cF_2$ form an orthonormal basis of $\bR^X$ w.r.t.\ the inner product $\langle \cdot,\cdot\rangle_\mu$. This implies
\begin{equation}
\label{eq:tight-duality-ball}
B = \left\{{\sum}_{f\in\cF_2}r_ff: r_f \in [-\varepsilon,\varepsilon] \ \textnormal{for all}\ f\in\cF_2\right\}.
\end{equation}

Let $\cF'_1\subseteq\cF_1$ be an $\varepsilon$-covering of $\cF_1$ w.r.t.\ norm $\|\cdot\|_{\mu,\cF_2}$ such that $|\cF_1'| = N_{\mu,\cF_2}(\cF_1,\varepsilon)$. By the definition of $\varepsilon$-covering,
\[
\cF_1\subseteq {\bigcup}_{f\in\cF_1'}(\{f\} + B).
\]
For a function class $\cF\subseteq\bR^X$, we define $\vect(\cF) = \{\vect(f):f\in\cF\}\subseteq\bR^m$. Now we have $\vect(\cF_1)\subseteq \bigcup_{f\in\cF_1'}\vect(\{f\} + B)$, which implies that the volume of $\vect(\cF_1)$ is at most $|\cF_1'|$ times the volume of $\vect(B)$.

It is clear that $\vect(\cF_1) = [-1,1]^m$ has volume $2^m$. By \eqref{eq:tight-duality-ball}, 
\[
\vect(B) = \left\{{\sum}_{v\in H_m}r_vv: r_v \in [-\varepsilon,\varepsilon] \ \textnormal{for all}\ v \in H_m\right\}.
\]
Since the columns of $H_m$ are orthogonal and have $\ell_2$ norm $\sqrt m$ in $\bR^m$, the volume of $\vect(B)$ is $(2\varepsilon\sqrt m)^m$.
Therefore,
\[
2^m \le |\cF_1'|\cdot (2\varepsilon \sqrt m)^m,
\]
and thus $|\cF_1'| \ge (1 /\varepsilon \sqrt m)^m$. Now we have
\begin{equation}
\label{eq:tight-switch-1}
\log N_{\mu,\cF_2}(\cF_1,\varepsilon) = \log |\cF_1'| \ge m\log (1/\varepsilon \sqrt m) \ge \Omega(m) \ge \Omega(1 / \varepsilon^2),
\end{equation}
and
\begin{equation}
\label{eq:tight-switch-2}
\log |\cF_2| = \log m \le O(\log (1/\varepsilon)).
\end{equation}
Combining \eqref{eq:tight-switch-1} and \eqref{eq:tight-switch-2} proves the lemma.
\end{proof}

\section{Sample Complexity of Distribution-Free OI}
\label{sec:distribution-free}
In this section, we consider the distribution-free setting where the OI learner has no knowledge of the distribution $\mu$, and its performance is measured on the worst-case distribution $\mu$. We focus on the case where $\cP = [0,1]^X$ and characterize the sample complexity of OI in this setting using the fat-shattering dimension of $\cD$ (\theoremref{thm:characterization-free}). Without the assumption that $\cP = [0,1]^X$, we give a sample complexity \emph{upper bound} for realizable OI in \remarkref{remark:general-P-D} and leave the intriguing question of sample complexity \emph{characterization} to future work.
\begin{theorem}
\label{thm:characterization-free}
For every distinguisher class $\cD$, every advantage bound $\varepsilon \in(0,1)$, and every failure probability bound $\delta\in(0,1)$,
the following sample complexity characterization holds for distribution-free OI:
\begin{align}
& \fat_\cD(12\varepsilon) / 8 + \log(1 - \delta)\notag \\
\le {} &
\sampdfr([0,1]^X, \cD,\varepsilon,\delta)\notag \\
= {} & \sampdfa([0,1]^X, \cD,\varepsilon,\delta) \notag \\
\le {} &
O\Big(\varepsilon^{-4}\Big(\fat_\cD(\varepsilon/25)(\log(2/\varepsilon))^2 + \log(2/\delta)\Big)\Big).\label{eq:characterization-free-upper}
\end{align}
\end{theorem}
This theorem is a direct corollary of the sample complexity upper bound (\lemmaref{lm:distribution-free-upper}) and lower bound (\lemmaref{lm:distribution-free-lower}) we prove in the remaining of this section.\footnote{
If we assume $\ell:=\sup_{x\in X}\sup_{d\in\cD}|d(x)| > 3\varepsilon$ and $\delta\in (0,1/3)$,
it is rather straightforward to show a sample complexity lower bound of $\Omega(\ell^2\varepsilon^{-2}\log(1/\delta))$ by a reduction from estimating the bias of a coin from independent tosses.
We omit the details as our focus is on the dependence on $\cD$, rather than on $\varepsilon$ and $\delta$.
}
With some modifications,
\theoremref{thm:characterization-free} also extends to \emph{multicalibration} (see Remarks~\ref{remark:multicalibration-upper} and \ref{remark:multicalibration-lower}).
\begin{remark}
\label{remark:general-P-D}
If we drop the assumption that $\cP = [0,1]^X$ and assume instead that $\cD = [-1,1]^X$, by \lemmaref{lm:advantage-l1}, our error objective $\adv_{\mu,\cD
}(\cdot,\cdot)$ becomes the $\ell_1$ error.
By the ideas and results in \citep{MR1408000,bartlett1995more}, it holds that
\[
\sampdfr(\cP, [-1,1]^X,\varepsilon,\delta) \le O\Big(\varepsilon^{-4}\Big(\fat_\cP(\varepsilon^2/20)(\log(2/\varepsilon))^2 + \log(2/\delta)\Big)\Big).
\]
Combining this with \eqref{eq:characterization-free-upper} and using the monotonicity \eqref{eq:samp-monotone-1} of the sample complexity of realizable OI, without assuming $\cP = [0,1]^X$ or $\cD = [-1,1]^X$, we have
\[
\sampdfr(\cP, \cD, \varepsilon,\delta) \le O\bigg(
\varepsilon^{-4}\Big(\min\{\fat_\cD(\varepsilon/25), \fat_\cP(\varepsilon^2/20)\}(\log(2/\varepsilon))^2 + \log(2/\delta)\Big)\bigg).
\]
This, however, does not give a sample complexity characterization for distribution-free realizable OI because when $\fat_\cD(\varepsilon/25)$ and $\fat_\cP(\varepsilon^2/20)$ are both infinite, it is still possible to have $\sampdfr(\cP, \cD, \varepsilon,\delta)$ being finite (see the example in \sectionref{sec:characterization}).
\end{remark}
\subsection{Upper Bound}
We prove our upper bound using 
the \emph{multiaccuracy boost} algorithm (\algorithmref{alg:iterative-updates}). \citet{kim2019multiaccuracy} used the algorithm to show a sample complexity upper bound for multiaccuracy based on an abstract notion of dimension of the distinguisher class. We make their upper bound concrete using the fat-shattering dimension and match it with a lower bound in \sectionref{sec:fat-lower}. The following standard uniform convergence result based on the fat-shattering dimension is crucial for our upper bound:

\begin{lemma}[Uniform convergence from fat shattering \citep{bartlett1995more}]
\label{lm:uniform-fat}
\ifdefined\preprint
\else
\\
\fi
Let $\cF\subseteq [-1,1]^X$ be a function class. 
For every $\varepsilon,\delta  \in (0,1)$ there exists $n\in\bZ_{>0}$ such that
\[
n = O\Big(\varepsilon^{-2}\Big(\fat_\cF(\varepsilon/5)(\log(2/\varepsilon))^2 + \log(2/\delta)\Big)\Big)
\]
and for every probability distribution $\mu$ over $X$,
\[
\Pr\left[\sup_{f\in \cF}\left|\frac 1n\sum_{i=1}^n f(x_i) - E_f\right| \ge \varepsilon\right] \le \delta,
\]
where $x_1,\ldots,x_n$ are drawn i.i.d.\ from $\mu$ and $E_f:=\bE_{x\sim\mu}[f(x)]$ for all $f\in \cF$.
\end{lemma}

\begin{lemma}
\label{lm:uniform-OI}
Let $\cD\subseteq [-1,1]^X$ be a distinguisher class. For every $\varepsilon,\delta\in(0,1)$ there exists $n\in \bZ_{>0}$ such that 
\[
n = O\Big(\varepsilon^{-2}\Big(\fat_\cD(\varepsilon/5)(\log(2/\varepsilon))^2 + \log(2/\delta)\Big)\Big)
\]
and for every distribution $\mu$ over $X$ and every predictor $p\in[0, 1]^X$,
\[
\Pr\left[\sup_{d\in\cD} \left|\frac 1n \sum_{i=1}^n o_id(x_i) - \langle p, d\rangle_\mu\right|\ge \varepsilon\right] \le \delta,
\]
where $(x_1,o_1),\ldots,(x_n,o_n)$ are drawn i.i.d.\ from $\mu_p$.
\end{lemma}
\begin{proof}
For every $d\in\cD$, define $\tilde d:X\times\{0,1\}\rightarrow [-1,1]$ to be a function that maps $(x,o)\in X\times\{0,1\}$ to $od(x)$. Define $\tilde\cD = \{\tilde d: d\in\cD\}$. 

We show that $\fat_{\tilde\cD}(\varepsilon) = \fat_\cD(\varepsilon)$ for all $\varepsilon >0$. Consider $(x_1,o_1),\ldots,(x_m,o_m)$ that are $\varepsilon$-fat shattered by $\tilde \cD$. Since every $\tilde d$ in $\tilde \cD$ maps $(x,0)$ to $0$ for all $x\in X$, we know that $o_1 = \cdots = o_m = 1$. This then implies that $x_1,\ldots,x_m$ are $\varepsilon$-fat shattered by $\cD$. Conversely, if $x_1,\ldots,x_m$ are $\varepsilon$-fat shattered by $\cD$, it is clear that $(x_1,1),\ldots,(x_m,1)$ are $\varepsilon$-fat shattered by $\tilde\cD$.

The lemma then follows directly from applying \lemmaref{lm:uniform-fat} with $\cF = \tilde\cD$.
\end{proof}

Now we state our sample complexity upper bound for distribution-free OI (\lemmaref{lm:distribution-free-upper}) and prove it using \algorithmref{alg:iterative-updates} and a helper lemma (\lemmaref{lm:iterative-updates}).
\begin{lemma}
\label{lm:distribution-free-upper}
For every distinguisher class $\cD$, every advantage bound $\varepsilon\in(0,1)$, and every failure probability bound $\delta\in(0,1)$,
the following  sample complexity upper bound holds for distribution-free OI:
\begin{align*}
\sampdfr([0,1]^X, \cD,\varepsilon,\delta) & = \sampdfa([0,1]^X, \cD,\varepsilon,\delta)\\
& \le
O\Big(\varepsilon^{-4}\Big(\fat_\cD(\varepsilon/25)(\log(2/\varepsilon))^2 + \log(2/\delta)\Big)\Big).
\end{align*}
\end{lemma}

\begin{algorithm2e}
  \LinesNumbered
  \SetKwInOut{KwPa}{Parameters}
  \SetKwInOut{KwIn}{Input}
  \SetKwInOut{KwOut}{Output}
  \KwPa{distinguisher class $\cD$, advantage bound $\varepsilon\in(0,1)$.}
  \KwIn{examples $(x_1,o_1),\ldots,(x_n,o_n)$.}
  \KwOut{predictor $p\in [0,1]^X$. }
  Initialize $p(x) = 1/2$ for all $x\in X$\;
  $T \gets \lceil 16/\varepsilon^2 \rceil$\;
  $m \gets \lfloor n / T\rfloor$\;
  \For{$t = 1,\ldots,T$}{
    take fresh examples $(x_1^*,o_1^*),\ldots,(x_m^*, o_m^*)$ where $(x_i^*,o_i^*) = (x_{(t - 1)m + i}, o_{(t-1)m + i)})$\;
  	draw $o_i'$ from $\ber(p(x_i^*))$ independently for all $i = 1,\ldots,n$\;
  	\eIf{there exists $d\in\cD\cup(-\cD)$ such that $\frac 1m\sum_{i=1}^m d(x_i^*)(o_i^* - o'_i) \ge 3\varepsilon /5$\label{line:1}
	}{
		$p(x)\gets p(x) + \varepsilon d(x)/5$ for all $x\in X$\label{line:2}\;
		$p(x) \gets \max\{0,\min\{1, p(x)\}\}$ for all $x\in X$\label{line:3}\;
	}{
	\Return $p$\label{line:4}\;
	}
  }{
  \Return $p$\label{line:5}\;
  }
  \caption{Multiaccuracy Boost \citep{hebert2018multicalibration,kim2019multiaccuracy}}
  \label{alg:iterative-updates}
\end{algorithm2e}
\begin{lemma}
\label{lm:iterative-updates}
If $\langle p^* - p, d\rangle_{\mu} \ge \varepsilon/5$ for some predictor $p^*\in[0,1]^X$ and some distribution $\mu\in\Delta_X$ before Line~\ref{line:2} of \algorithmref{alg:iterative-updates} is executed, then Lines~\ref{line:2}-\ref{line:3} decrease
$\|p - p^*\|_2^2
$
by at least $\varepsilon^2/25$. Here, we use $\|f\|_2^2$ as a shorthand for $\langle f, f\rangle_\mu$.
\end{lemma}
\begin{proof}
Line~\ref{line:2} decreases $\|p^* - p\|_2^2$ by at least $\varepsilon ^2 /25$ because
\begin{align*}
& \|p^* - p\|_2^2 - \|p^* - (p + (\varepsilon  / 5)d)\|_2^2\\
= {} & 2 \langle p^* - p, (\varepsilon / 5) d \rangle_\mu - (\varepsilon /5)^2\|d\|_2^2\\
\ge {} & 2(\varepsilon / 5)^2 - (\varepsilon / 5)^2\\
= {} & \varepsilon ^ 2 / 25.
\end{align*}
The lemma is proved by noting that Line~\ref{line:3} never increases $\|p^* - p\|_2^2$.
\end{proof}
Now we finish the proof of \lemmaref{lm:distribution-free-upper}.
\begin{proof}
We first consider the relatively simple case where $\fat_\cD(\varepsilon/25) = 0$. By the definition of the fat-shattering dimension, there exists $d^*\in[-1,1]^X$ such that 
\begin{equation}
\label{eq:fat=0-1}
|d^*(x) - d(x)| \le \varepsilon/25\quad \textnormal{for every $d\in \cD$ and every $x\in X$.}
\end{equation}
Therefore, for every fixed distribution $\mu\in\Delta_X$ and target predictor $p^*\in [0,1]^X$, as long as the learned predictor $p$ satisfies $|\langle p - p^*,d^*\rangle_\mu| \le \varepsilon/2$, we get the desired error bound $\|p - p^*\|_{\mu,\cD} \le \varepsilon$. Decomposing $d^* = d^*_+ + d^*_-$ with $d^*_+(x) := \max\{d^*(x), 0\}$ and $d^*_-(x) := \min\{d^*(x),0\}$ for every $x\in X$, we aim for the stronger goal that $|\langle p - p^*,d^*_+\rangle_\mu| \le \varepsilon/4$ and $|\langle p - p^*,d^*_-\rangle_\mu| \le \varepsilon/4$.

Given input examples $(x_1,o_1),\ldots,(x_n,o_n)$ drawn i.i.d.\ from $\mu_{p^*}$,  we first compute $u_+:= \frac 1n\sum_{i=1}^n d^*_+(x_i)o_i$ and $v_+:= \frac 1n\sum_{i=1}^n d^*_+(x_i)$. It is clear that $0 \le u_+ \le v_+$, so there exists $r_+\in [0,1]$ such that $u_+ = r_+v_+$. Similarly, we compute $u_-,v_-$ satisfying $0 \ge u_- \ge v_-$ and define $r_-\in [0,1]$ so that $u_- = r_-v_-$. We output $p\in [0,1]^X$ with $p(x) = r_+$ if $d^*(x) \ge 0$ and $p(x) = r_-$ otherwise.

By the Chernoff bound and the union bound, we can choose $n = O(\varepsilon^{-2}\log(2/\delta))$ such that with probability at least $1-\delta/2$, both of the following inequalities hold:
\begin{align}
|u_+ - \langle p^*, d^*_+\rangle_\mu| \le \varepsilon/8,\label{eq:fat=0-2}\\
|v_+ - \bE_{x\sim\mu}[d^*_+(x)]| \le \varepsilon/8.\label{eq:fat=0-3}
\end{align}
When multiplied by $r_+$, \eqref{eq:fat=0-3} implies $|u_+ - \langle p, d^*_+ \rangle_\mu| \le \varepsilon/8$. Combining this with \eqref{eq:fat=0-2}, with probability at least $1-\delta/2$, $|\langle p^* - p,d^*_+\rangle_\mu| \le \varepsilon/4$. Similarly, with probability at least $1-\delta/2$, $|\langle p^* - p,d^*_-\rangle_\mu| \le \varepsilon/4$. 
By the union bound, with probability at least $1-\delta$, $|\langle p^* - p,d^*_+\rangle_\mu| \le \varepsilon/4$ and $|\langle p^* - p,d^*_-\rangle_\mu| \le \varepsilon/4$ both hold as desired, so
\begin{align*}
\sampdfr([0,1]^X, \cD,\varepsilon,\delta) & \le O(\varepsilon^{-2}\log(2/\delta))\\
& \le O\Big(\varepsilon^{-4}\Big(\fat_\cD(\varepsilon/25)(\log(2/\varepsilon))^2 + \log(2/\delta)\Big)\Big).
\end{align*}

We thus assume $\fat_\cD(\varepsilon/25)\ge 1$ from now on.
We use \algorithmref{alg:iterative-updates} with $n  = \lceil 25/\varepsilon^2\rceil m$ for a positive integer $m$ chosen as follows. 
Given input examples $(x_1,o_1),\ldots,(x_n,o_n)$ drawn i.i.d.\ from $\mu_{p^*}$ for some $\mu\in\Delta_X$ and $p^*\in [0,1]^X$,
by \lemmaref{lm:uniform-OI} and the union bound, we can choose 
\[
m = O\Big(\varepsilon^{-2}\Big(\fat_\cD(\varepsilon/25)(\log(2/\varepsilon))^2 + \log(2/\delta\varepsilon)\Big)\Big)
\]
so that with probability at least $1-\delta$, every time Line~\ref{line:1} is executed, we have
\begin{align}
& \sup_{d\in\cD}\left|\frac 1m\sum_{i=1}^m d(x_i^*)o_i^* - \langle d, p^*\rangle_\mu\right| \le \varepsilon /5, \textnormal{and} \label{eq:free-upper-1}\\
& \sup_{d\in\cD}\left|\frac 1m\sum_{i=1}^m d(x_i^*)o_i' - \langle d, p\rangle_\mu\right| \le \varepsilon /5. \label{eq:free-upper-2}
\end{align}
Assuming this is true, when the output predictor $p$ is returned at Line~\ref{line:4}, we have
\[
\sup_{d\in\cD}|\langle p - p^*,d\rangle_\mu| \le 3\varepsilon/5 + \varepsilon/5 + \varepsilon / 5 \le\varepsilon,
\]
as desired. 
Moreover, \eqref{eq:free-upper-1} and \eqref{eq:free-upper-2} imply $\langle p^* - p,d\rangle_\mu \ge 3\varepsilon/5 - \varepsilon/5 - \varepsilon/5 = \varepsilon/5$ before every time Line~\ref{line:2} is executed, so
by \lemmaref{lm:iterative-updates},
when the output predictor $p$ is returned at Line~\ref{line:5}, 
\[
\|p - p^*\|_2^2 \le 1 - (\varepsilon^2/25)\lceil 25/\varepsilon^2\rceil \le 0,
\]
as desired. Therefore,
\begin{align*}
\sampdfr([0,1]^X, \cD,\varepsilon,\delta) & = \sampdfa([0,1]^X, \cD,\varepsilon,\delta)\\
& \le n\\
& = \lceil 25/\varepsilon^2\rceil m\\
& = O\Big(\varepsilon^{-4}\Big(\fat_\cD(\varepsilon/25)(\log(2/\varepsilon))^2 + \log(2/\delta\varepsilon)\Big)\Big)\\
& = O\Big(\varepsilon^{-4}\Big(\fat_\cD(\varepsilon/25)(\log(2/\varepsilon))^2 + \log(2/\delta)\Big)\Big).\tag{by $\fat_\cD(\varepsilon/25)\ge 1$}
\end{align*}
\end{proof}
\begin{remark}
\label{remark:multicalibration-upper}
\citet{hebert2018multicalibration} used a modified version of \algorithmref{alg:iterative-updates} for \emph{multicalibration}, and indeed, this gives a sample complexity upper bound for multicalibration similar to \lemmaref{lm:distribution-free-upper}. In multicalibration, the only difference is in the error objective: the interval $[0,1]$ is partitioned into $k$ subsets $\Lambda_1,\ldots,\Lambda_k$ for some $k\in\bZ_{>0}$, and given a distribution $\mu\in \Delta_X$ and a distinguisher class $\cD\subseteq[-1,1]^X$, we measure the error of a learned predictor $p\in [0,1]^X$ w.r.t.\ the target predictor $p^*\in [0,1]^X$ by 
\begin{equation}
\label{eq:mc-error}
\mcerror_{\mu,\cD}(p,p^*) :=
\sup_{d\in \cD,1\le j \le k}|\bE_{x\sim\mu}[(p(x) - p^*(x))d(x)\one(p(x)\in \Lambda_j)]|.
\end{equation}
For $\varepsilon,\delta\in (0,1)$, suppose
our goal is to achieve $\mcerror_{\mu,\cD}(p,p^*) \le \varepsilon$ with probability at least $1-\delta$ in the distribution-free setting. Changing Line~\ref{line:1} of \algorithmref{alg:iterative-updates} to
\begin{align*}
\text{``}
\textnormal{\bfseries if} \textit{ there exists } d\in\cD\cup (-\cD)\textit{ and } j\in \{1,\ldots,k\} 
\textit{ such that }\\
\frac 1m\sum_{i=1}^md(x_i^*)(o_i^* - o_i')\one(p(x_i^*)\in \Lambda_j) \ge 3\varepsilon/5 \textnormal{ \bfseries then}\text{''}
\end{align*}
and changing Line~\ref{line:2} to
\[
\text{``}
p(x) \gets p(x) + \varepsilon d(x)\one(p(x)\in \Lambda_j)/5 \textnormal { for all }x\in X\text{''},
\]
we can prove a sample complexity upper bound of
\[
O\Big(\varepsilon^{-4}\Big(\fat_\cD(\varepsilon/25)(\log(2/\varepsilon))^2 + \log(1/\delta)\Big)\Big).
\]
This bound follows from combining the proof of \lemmaref{lm:distribution-free-upper} with the observation that for every predictor $p\in [0,1]^X$, the following distinguisher class 
\begin{align*}
\cD_p := \{d'\in [-1,1]^X:\exists d\in \cD \text{ and } j\in \{1,\ldots,k\}\text{ such that}\\ \text{for every }x\in X,d'(x) = d(x)\one(p(x)\in \Lambda_j)\}
\end{align*}
satisfies $\fat_{\cD_p}(\varepsilon) \le \fat_{\cD}(\varepsilon) + 1$ for every $\varepsilon>0$.
\end{remark}
\subsection{Lower Bound}
\label{sec:fat-lower}
We first prove the following lemma showing that the fat-shattering dimension gives a lower bound for the metric entropy on a particular distribution $\mu$.
\begin{lemma}
\label{lm:fat-covering}
If $x_1,\ldots,x_n\in X$ are $6\varepsilon$-fat shattered by $\cD$, then
$\log N_{\mu, \cD}([0,1]^X, \varepsilon) \ge n/8$, where $\mu$ is the uniform distribution over $x_1,\ldots,x_n$.
\end{lemma}
\begin{proof}
We can assume without loss of generality that $X = \{x_1,\ldots,x_n\}$.
The assumption that $x_1,\ldots,x_n$ are $6\varepsilon$-fat shattered by $\cD$ implies the existence of a function $r\in\bR^X$ such that for every function $b \in\{-1,1\}^X$, there exists a distinguisher $d_b\in\cD$ satisfying 
\begin{equation}
\label{eq:fat-covering-shattering}
b(x)(d_b(x) - r(x)) \ge 6\varepsilon\quad \text{for all}\ x\in X.
\end{equation}

We first show that $\{r\} + [-6\varepsilon, 6\varepsilon]^X$ is a subset of the symmetric convex hull $\bar \cD$. 
Assume for the sake of contradiction that some $f\in \{r\} + [-6\varepsilon, 6\varepsilon]^X$ does not belong to $\bar\cD$.
In particular, $f$ does not belong to the symmetric convex hull of $\{d_b:b\in\{-1,1\}^X\}$. 
By the hyperplane separation theorem, there exists $g\in \bR^X$ such that 
\begin{equation}
\label{eq:hyperplane-separation}
\langle g, d_b - f\rangle_\mu < 0\quad \text{for all}\ b\in\{-1,1\}^X. 
\end{equation}
Consider the function $b\in\{-1,1\}^X$ such that $b(x) = 1$ if and only if $g(x) \ge 0$. For $x\in X$ with $b(x) = 1$, we have $d_b(x) - r(x) \ge 6\varepsilon$ by \eqref{eq:fat-covering-shattering} and thus $d_b(x) \ge f(x)$. Similarly, for $x\in X$ with $b(x) = -1$ we have $d_b(x) \le f(x)$. In both cases, we have $g(x)(d_b(x) - f(x)) \ge 0$, a contradiction with \eqref{eq:hyperplane-separation}.

Now we have proved that $\{r\} + [-6\varepsilon,6\varepsilon]^X\subseteq \bar\cD$. By the symmetry of $\bar\cD$, $\{-r\} + [-6\varepsilon,6\varepsilon]^X\subseteq \bar\cD$. Then by the convexity of $\bar\cD$, $[-6\varepsilon,6\varepsilon]^X\subseteq \bar\cD$.

The lemma is proved by the following chain of inequalities:
\begin{align}
\log N_{\mu, \cD}([0,1]^X, \varepsilon) & = \log N_{\mu,\bar\cD}([-1,1]^X, 2\varepsilon) \label{eq:distribution-free-lower-1}\\
& \ge \log N_{\mu,[-6\varepsilon,6\varepsilon]^X}([-1,1]^X, 2\varepsilon) \label{eq:distribution-free-lower-2}\\
& = \log N_{\mu,[-1,1]^X}([-1, 1]^X, 1/3) \label{eq:distribution-free-lower-3}\\
& \ge n/8. \label{eq:distribution-free-lower-4}
\end{align}
We used \lemmaref{lm:covering-basics} Items \ref{item:basics-hull}, \ref{item:basics-shift}, and \ref{item:basics-homogeneity} in \eqref{eq:distribution-free-lower-1}, 
\lemmaref{lm:covering-basics} Item \ref{item:basics-monotone} in
\eqref{eq:distribution-free-lower-2}, and 
\lemmaref{lm:covering-basics} Item \ref{item:basics-homogeneity} in 
\eqref{eq:distribution-free-lower-3}. We used \lemmaref{lm:helper} to get \eqref{eq:distribution-free-lower-4}.
\end{proof}
Combining the lemma above with \lemmaref{lm:OI-sample-lower}, we obtain the following sample complexity lower bound for distribution-free OI. 
\begin{lemma}
\label{lm:distribution-free-lower}
For every distinguisher class $\cD$, every advantage bound $\varepsilon > 0$, and every failure probability bound $\delta\in(0,1)$,
the following sample complexity lower bound holds for distribution-free OI:
\begin{align*}
\sampdfr([0,1]^X, \cD,\varepsilon,\delta) & = \sampdfa([0,1]^X, \cD,\varepsilon,\delta)\\
& \ge
\fat_\cD(12\varepsilon) / 8 + \log(1 - \delta).
\end{align*}
\end{lemma}
\begin{proof}
Define $n = \fat_\cD(12\varepsilon)$, and suppose $x_1,\ldots,x_n\in X$ are $12\varepsilon$-shattered by $\cD$. Let $\mu$ be the uniform distribution over $x_1,\ldots,x_n$. By \lemmaref{lm:fat-covering}, $\log N_{\mu,\cD}([0,1]^X,2\varepsilon) \ge n/8$. By \lemmaref{lm:OI-sample-lower},
\begin{align}
\sampdfa([0,1]^X, \cD,\varepsilon,\delta) & = \sampdfr([0,1]^X, \cD,\varepsilon,\delta) \notag\\
& \ge \sampdsr([0,1]^X,\cD,\varepsilon,\delta,\mu) \tag{by \eqref{eq:samp-monotone-2}}\\
& \ge \log N_{\mu,\cD}([0,1]^X,2\varepsilon) + \log(1-\delta) \notag \\
& \ge n/8 + \log(1-\delta). \notag
\ifdefined\preprint
\qedhere
\fi
\end{align}
\end{proof}
\begin{remark}
\label{remark:multicalibration-lower}
It is clear that the error objective $\mcerror_{\mu,\cD}(p,p^*)$ for multicalibration defined in \eqref{eq:mc-error} satisfies
\[
\mcerror_{\mu,\cD}(p,p^*) \ge \adv_{\mu,\cD}(p,p^*)/k,
\]
so \lemmaref{lm:distribution-free-lower} directly implies a sample complexity lower bound for multicalibration. Specifically, assuming the predictor class $\cP$ is $[0,1]^X$, if we want to achieve $\mcerror_{\mu,\cD}(p,p^*) \le \varepsilon$ with probability at least $1-\delta$ in the distribution-free setting, the sample complexity is at least
\[
\fat_\cD(12k\varepsilon)/8 + \log(1 - \delta).
\]
\end{remark}

\section{Separation between Agnostic and Realizable OI}
\label{sec:separation}

The sample complexities of realizable and agnostic learning have the same characterization in many settings.
They are both characterized by the VC dimension in distribution-free PAC learning \citep{MR3408730},
whereas in distribution-specific PAC learning
they are both characterized by the metric entropy (this characterization was proved in the realizable setting by \citet{MR1122796}, and it extends straightforwardly to the agnostic setting (see \lemmaref{lm:binary-l1})). Recently, an independent work by \cite{hopkins2021realizable} shows that this relationship between realizable and agnostic learning holds very broadly for a unified reason.

In this section, we study this relationship between realizable and agnostic learning in outcome indistinguishability. In contrast to many common learning settings including PAC learning, we show a strong separation between the sample complexities of realizable OI and agnostic OI in both the distribution-free and the distribution-specific settings.

\subsection{No separation with additional assumptions}
Before we present the sample complexity separation between realizable and agnostic OI, we first discuss several assumptions that make the separation disappear in the following two lemmas. 

In the first lemma (\lemmaref{lm:inside-hull}), we make the assumption that the target predictor $p^*$ belongs to the symmetric convex hull $\bar\cP$ of $\cP$.

\begin{lemma}[Inside the symmetric convex hull]
\label{lm:inside-hull}
For every predictor class $\cP\subseteq[0,1]^X$, every distinguisher class $\cD\subseteq[-1,1]^X$, every distribution $\mu\in\Delta_X$, every advantage bound $\varepsilon >0$, and every failure probability bound $\delta \in(0,1)$,
there exist a nonnegative integer $n$ and an algorithm $\cA$ such that
\begin{align*}
n & = O(\varepsilon^{-2}(\log N_{\mu,\cP}(\cD,\varepsilon/4) + \log(2/\delta))),\quad\textnormal{and}\\
\cA & \in {\bigcap}_{p^*\in\bar\cP}\oi_n(p^*,\cD,\varepsilon,\delta,\mu).
\end{align*}
\end{lemma}
\begin{proof}
By \lemmaref{lm:covering-basics} Item \ref{item:basics-hull}, 
\[\log N_{\mu,\cP}(\cD,\varepsilon/4))
=
\log N_{\mu,\bar\cP}(\cD,\varepsilon/4)).
\]
The lemma is then a consequence of applying \lemmaref{lm:OI-sample-upper} with $\cP$ replaced by $\bar\cP$.
\end{proof}

In the second lemma (\lemmaref{lm:binary-l1}), we make the assumption that $\{-1,1\}\subseteq \cD$ (i.e.\ we consider the $\ell_1$ error, see \lemmaref{lm:advantage-l1}), and that either $\cP$ only contains binary classifiers or the target predictor $p^*$ is a binary classifier.
\begin{lemma}[Binary classifiers with $\ell_1$ error]
\label{lm:binary-l1}
Assume $\{-1,1\}^X\subseteq\cD$.
For every predictor class $\cP\subseteq[0,1]^X$, every distribution $\mu\in\Delta_X$, every advantage bound $\varepsilon \in (0,1)$ and every failure probability bound $\delta\in(0,1)$,
there exists a positive integer $n$ such that
\[
n = O((\log N_{\mu,\cD}(\cP, \varepsilon/2) + \log(2/\delta))/\varepsilon^2),
\]
and for every target predictor $p^*\in[0,1]^X$,
\algorithmref{alg:erm} belongs to
\[
\oi_n(p^*,\cD, \inf_{p\in\cP}\adv_{\mu,\cD}(p,p^*) + \varepsilon,\delta,\mu)
\]
whenever $\cP\subseteq\{0,1\}^X$ or $p^*\in\{0,1\}^X$.
\end{lemma}

\begin{proof}
We choose $n = O((\log N_{\mu,\cD}(\cP, \varepsilon/2) + \log(2/\delta))/\varepsilon^2)$ so that by the Chernoff bound and the union bound, with probability at least $1 - \delta$ over $(x_1,o_1),\ldots,(x_n,o_n)$ drawn i.i.d.\ from $\mu_{p^*}$, for all $p'\in\cP'$,
\begin{equation}
\label{eq:binary-l1-1}
\left|\frac 1n\sum_{i=1}^n|p'(x_i) - o_i| - \bE_{(x,o)\sim\mu_{p^*}}[|p'(x) - o|]\right| \le \varepsilon/8.
\end{equation}

It remains to prove that whenever the input $(x_1,o_1),\ldots,(x_n,o_n)$ to \algorithmref{alg:erm} satisfies the inequality above, the output $p$ satisfies $\adv_{\mu,\cD}(p,p^*) \le \inf_{p\in\cP}\adv_{\mu,\cD}(p,p^*) + \varepsilon$ assuming $\cP\subseteq\{0,1\}^X$ or $p^*\in\{0,1\}^X$.

By definition, there exists $\hat p\in\cP$ such that $\|\hat p - p^*\|_{\mu,\cD} \le \inf_{p\in\cP}\|p - p^*\|_{\mu,\cD} + \varepsilon/4$. Since $\cP'$ is an $\varepsilon/2$-covering of $\cP$, there exists $\tilde p\in\cP'$ such that $\|\tilde p - \hat p\|_{\mu,\cD} \le \varepsilon/2$. Combining these two inequalities together,
\begin{equation}
\label{eq:binary-l1-2}
\|\tilde p - p^*\|_{\mu,\cD} \le \|\tilde p - \hat p\|_{\mu,\cD} + \|\hat p - p^*\|_{\mu,\cD} \le \inf_{p\in\cP}\|p - p^*\|_{\mu,\cD} + 3\varepsilon/4.
\end{equation}
Since $\{-1,1\}^X\subseteq \cD$, 
\[
\loss(p') = \frac 1n\sum_{i=1}^n|p'(x_i) - o_i| \quad \textnormal{for all}\ p'\in\cP'.
\]
Therefore, the output $p$ of \algorithmref{alg:erm} satisfies
\begin{equation}
\label{eq:binary-l1-3}
\frac 1n\sum_{i=1}^n|p(x_i) - o_i|  = \loss(p) \le \loss(\tilde p)  = \frac 1n\sum_{i=1}^n|\tilde p(x_i) - o_i|.
\end{equation}
Our assumption $\cP\subseteq\{0,1\}^X$ or $p^*\in\{0,1\}^X$ implies that for all $p'\in\cP'$,
\begin{equation}
\label{eq:binary-l1-4}
\bE_{(x,o)\sim \mu_{p^*}}[|p'(x) - o|] = 
\bE_{x\sim\mu}[|p'(x) - p^*(x)|] = 
\|p' - p^*\|_{\mu,\cD},
\end{equation}
where the last equation is by \lemmaref{lm:advantage-l1}.

Combining everything together,
\begin{align*}
\|p - p^*\|_{\mu,\cD} & \le \frac{1}{n}\sum_{i=1}^n|p(x_i) - o_i| + \varepsilon/8 \tag{by \eqref{eq:binary-l1-1} and \eqref{eq:binary-l1-4}}\\
& \le \frac{1}{n}\sum_{i=1}^n|\tilde p(x_i) - o_i| + \varepsilon/8 \tag{by \eqref{eq:binary-l1-3}}\\
& \le \|\tilde p - p^*\|_{\mu,\cD} + \varepsilon/4 \tag{by \eqref{eq:binary-l1-1} and \eqref{eq:binary-l1-4}}\\
& \le \inf_{p\in\cP}\|p - p^*\|_{\mu,\cD} + \varepsilon. \tag{by \eqref{eq:binary-l1-2}}
\end{align*}
\end{proof}

\subsection{Separation without additional assumptions}
\label{sec:real-separation}
Without the additional assumptions in \lemmaref{lm:inside-hull} and \lemmaref{lm:binary-l1}, 
we give examples where the sample complexity of agnostic OI can be arbitrarily larger than that of realizable OI in both the distribution-specific and the distribution-free settings.
Given a positive integer $m$, we choose $X = \{\bot\}\cup\{0,1\}^m$, and choose the predictor class $\cP$ to consist of only two predictors $p_1$ and $p_2$ where $p_1(\bot) = 0,p_2(\bot) = 1$ and $p_1(x) = p_2(x) = 1/2$ for all $x\in\{0,1\}^m$.

We first show that $O(\varepsilon^{-2}\log (2/\delta))$ examples are sufficient for distribution-free agnostic OI as long as $\cD = [-1,1]^X$. 
\begin{lemma}
\label{lm:easy-OI}
For all $\varepsilon,\delta \in(0,1)$, there exists a positive integer $n = O(\varepsilon^{-2}\log (2/\delta))$ such that for all $m$ and $X,\cP$ defined as above, there exists an algorithm $\cA\in \dfa_n(\cP,[-1,1]^X,\varepsilon,\delta)$.
\end{lemma}
\begin{proof}
We choose algorithm $\cA$ to be the following simple algorithm: after seeing examples $(x_1,o_1),\allowbreak\ldots,\allowbreak(x_n,o_n)$, it 
computes 
\[
r := \frac{|\{i\in\{1,\ldots,n\}:(x_i,o_i) = (\bot, 1)\}|}{|\{i\in\{1,\ldots,n\}:x_i = \bot\}|}.
\]
If the denominator $|\{i\in\{1,\ldots,n\}:x_i = \bot\}|$ is zero, define $r = 1/2$. Algorithm $\cA$ outputs the predictor $p$ such that $p(\bot) = r$, and $p(x) = 1/2$ for all $x\in\{0,1\}^m$.

It remains to show that when the input examples $(x_1,o_1),\ldots,(x_n,o_n)$ are drawn i.i.d.\ from $\mu_{p^*}$ for some distribution $\mu\in\Delta_X$ and some $p^*\in[0,1]^X$, the output $p$ of algorithm $\cA$ satisfies the following with probability at least $1 - \delta$:
\begin{equation}
\label{eq:easy-OI-1}
\|p - p^*\|_{\mu,[-1,1]^X} \le \inf_{p'\in\cP}\|p'-p^*\|_{\mu,[-1,1]^X} + \varepsilon.
\end{equation}

Let $\mu_\bot$ denote the probability mass on $\bot$ in distribution $\mu$. 
Since $p(x) = 1/2 = p'(x)$ for all $p'\in\cP$ and $x\in\{0,1\}^X$,
by \lemmaref{lm:advantage-l1},
\begin{equation}
\label{eq:easy-OI-2}
\|p - p^*\|_{\mu,[-1,1]^X} - \mu_\bot|p(\bot) - p^*(\bot)| \le \inf_{p'\in\cP}\|p'-p^*\|_{\mu,[-1,1]^X}.
\end{equation}

If $\mu_\bot \le \varepsilon$, inequality \eqref{eq:easy-OI-2} implies that \eqref{eq:easy-OI-1} is always true. 
If $\mu_\bot > \varepsilon$,
we choose $n = O(\varepsilon^{-2}\log(2/\delta))$ so that the following two conditions hold:
\begin{enumerate}
\item by the Chernoff bound, with probability at least $1 - \delta/2$, $|\{i\in\{1,\ldots,n\}:x_i = \bot\}| \ge \mu_\bot n/2$;
\item it holds that $\mu_{\bot}n/2 \ge C \mu_\bot\varepsilon^{-2}\log(2/\delta)$, where $C$ is an absolute constant determined later.
\end{enumerate}
When combined, the two conditions guarantee that with probability at least $1-\delta/2$, 
$|\{i\in\{1,\ldots,n\}:x_i = \bot\}|\ge C\mu_\bot\varepsilon^{-2}\log(2/\delta)$.
Conditioned on this being true, by the Chernoff bound, choosing $C$ sufficiently large guarantees that with probability at least $1 - \delta/2$,
\[
|p^*(\bot) - p(\bot)| = \left|p^*(\bot) - \frac{|\{i:(x_i,o_i) = (\bot, 1)\}|}{|\{i:x_i = \bot\}|}\right| \le \varepsilon/{\sqrt{\mu_\bot}}.
\]
Combining this with \eqref{eq:easy-OI-2}, we know \eqref{eq:easy-OI-1} holds with probability at least $1-\delta$, as desired.
\end{proof}

By the monotonicity properties of sample complexities (see \eqref{eq:samp-monotone-1} and \eqref{eq:samp-monotone-2}), the lemma above implies that for all $\cD\subseteq[-1,1]^X$ and $\mu\in\Delta_X$,
\begin{equation}
\sampdsa(\cP,[-1,1]^X,\varepsilon,\delta,\mu) \le \sampdfa(\cP,[-1,1]^X,\varepsilon,\delta) \le O(\varepsilon^{-2}\log(2/\delta)), \label{eq:easy-OI-monotone-1}\\
\end{equation}
and
\begin{align}
\sampdsr(\cP,\cD,\varepsilon,\delta,\mu) \le \sampdfr(\cP,\cD,\varepsilon,\delta) & \le \sampdfr(\cP,[-1,1]^X,\varepsilon,\delta) \notag \\ & \le O(\varepsilon^{-2}\log(2/\delta)). \label{eq:easy-OI-monotone-3}
\end{align}

Now we show that for a specific distribution $\mu$ and a particular distinguisher class $\cD$, it requires at least $m - 20$ examples for distribution-specific agnostic OI when $\varepsilon = 1/8$ and $\delta = 1/3$ (\lemmaref{lm:agnostic-lower}). Sending $m$ to infinity and comparing with \eqref{eq:easy-OI-monotone-3}, this implies that the sample complexity of agnostic OI can be arbitrarily large even when the sample complexity of realizable OI is bounded by a constant in both the distribution-specific and the distribution-free settings. Comparing this with \eqref{eq:easy-OI-monotone-1}, we also know that the sample complexity of agnostic OI is not monotone w.r.t.\ the distinguisher class $\cD$ in both the distribution-specific and distribution-free settings.

Before we describe the $\mu$ and $\cD$ used in \lemmaref{lm:agnostic-lower}, we need the following definitions.
Let $\bool$ denote the set of all functions $t:\{0,1\}^m\rightarrow\{0,1\}$. A function $f\in\bool$ is a parity function if for a subset $S\subseteq \{1,\ldots,m\}$, it holds that $f(x) = \bigoplus_{i\in S} x_i$ for all $x\in\{0,1\}^m$. A function $g\in\bool$ is an anti-parity function if for some parity function $f$, it holds that $g(x) = 1\oplus f(x)$ for all $x\in\{0,1\}^m$. 
Let $\parity\subseteq \bool$ (resp.\ $\anti\subseteq \bool$) denote the set of parity functions (resp.\ anti-parity) functions.

We choose $\mu$ so that it puts $1/3$ probability mass on $\bot$, and puts the remaining $2/3$ probability mass evenly on $\{0,1\}^m$. We choose $\cD$ to contain $2^m$ hypotheses as follows: for every parity function $f\in\parity$, there is a distinguisher $d\in \cD$ such that $d(\bot) = 1$ and $d(x) = (-1)^{f(x)}$ for all $x\in \{0,1\}^m$.
\begin{lemma}
\label{lm:agnostic-lower}
For $m,\cP,\cD,\mu$ defined as above,
$\sampdsa(\cP,\cD,1/8,1/3,\mu)\ge m - 20$.
\end{lemma}

To prove the lemma, we first consider the following task of list-learning parity and anti-parity functions, which we abbreviate as \listl. 
Given a positive integer list size $k$, in the task of \listl,
an algorithm takes as input examples
$(u_1,v_1),\ldots,(u_n,v_n)$ where $u_i$'s are drawn independently and uniformly from $\{0,1\}^m$, and $v_i = t(u_i)$ for some unknown $t\in\bool$, and the algorithm outputs a subset (i.e.\ list) $L\subseteq \bool$. The algorithm succeeds if $t\in L$ and $\min\{|L\cap \parity|, |L\cap \anti|\} \le k$.

\begin{lemma}
\label{lm:parity}
Assuming $n \le m$ and $k \le 2^{m - n}$,
no algorithm has failure probability smaller than $(1/2)(1 - 1/2^{m - n})(1 - k/2^{m - n})$ for all choices of $t\in\bool$ in the task \listl.
\end{lemma}
\begin{proof}
Assume that for some $\alpha\ge 0$,
there exists an algorithm $\cA$ that has failure probability at most $\alpha$ for all $t\in\bool$. In particular, when $t$ is drawn uniformly at random from $\parity \cup \anti$, the failure probability of $\cA$ is at most $\alpha$. For this fixed distribution of $t$, we can assume that algorithm $\cA$ is deterministic without increasing its failure probability. By the law of total probability,
\[
\alpha \ge \Pr[\textnormal{failure}]  = \bE[\Pr[\textnormal{failure}|(u_1,v_1),\ldots,(u_n, v_n)]].
\]
Here, $(u_1,v_1),\ldots,(u_n,v_n)$ are the input examples to algorithm $\cA$ where $u_1,\ldots,u_n$ are drawn independently and uniformly from $\{0,1\}^m$, and $v_i = t(u_i)$ for every $i = 1,\ldots,n$ with $t$ drawn uniformly at random from $\parity\cup\anti$.

With probability 
\[
(1 - 2^{-m})(1 - 2^{-m + 1})\cdots (1 - 2^{-m + n - 1}) \ge 1 - 2^{-m} - \cdots - 2^{-m + n - 1} \ge 1 - 1/2^{m-n},
\]
$u_1,\ldots,u_n$ are linearly independent in $\mathbb F_2^m$, in which case
the conditional distribution of $t$ given $(u_1,v_1),\ldots,(u_n,v_n)$ is the uniform distribution over $L_1\cup L_2$ for some $L_1\subseteq \parity,L_2\subseteq \anti$ with $|L_1| = |L_2| = 2^{m - n}$, and thus
\[
\Pr[\textnormal{failure}|(u_1,v_1),\ldots,(u_n,v_n)] \ge (1/2)(1 - k/2^{m - n}).
\]
Therefore,
\[
\alpha\ge
\bE[\Pr[\textnormal{failure}|(u_1,v_1),\ldots,(u_n, v_n)]] \ge (1/2)(1 - 1/2^{m - n})(1 - k/2^{m - n}),
\]
as desired.
\end{proof}
We are now ready to prove \lemmaref{lm:agnostic-lower}.
Let $\cA$ be an algorithm in $\dsa_n(\cP,\cD,1/8,\delta,\mu)$ for a nonnegative integer $n$ and a positive real number $\delta$. We use this algorithm to build an algorithm for \listl.
Suppose $(u_1,v_1),\ldots,(u_n,v_n)$ are the input examples in \listl, where $u_i$ are drawn independently and uniformly from $\{0,1\}^m$, and $v_i = t(u_i)$ for some $t\in\bool$. 

Before we construct the algorithm for \listl, we need the following definition.
For every $f\in\bool$, we define $p_f\in[0,1]^X$ such that $p_f(\bot) = 1/2$, and $p_f(x) = (1 + (-1)^{f(x)})/2$ for all $x\in\{0,1\}^m$.

Now we construct the algorithm for \listl\ using algorithm $\cA$.
We start by generating the input examples $(x_1,o_1),\ldots,(x_n, o_n)$ to algorithm $\cA$. Independently for every $i = 1,\ldots,n$, with probability $1/6$ we set $(x_i,o_i) = (\bot, 0)$, with probability $1/6$ we set $(x_i,o_i) = (\bot, 1)$ and with the remaining probability $2/3$ we set $(x_i,o_i) = (u_i, (1 + (-1)^{v_i})/2)$. After running algorithm $\cA$ and obtaining the output $p$, we return the list
\begin{equation}
\label{eq:def-L}
L = \{f\in\bool: \|p_f - p\|_{\mu,\cD} \le 1/6 + 1/8\}\cup(\bool\backslash(\parity\cup\anti)).
\end{equation}
We prove the following two helper lemmas before we finish the proof of \lemmaref{lm:agnostic-lower}.
\begin{lemma}
\label{lm:agnostic-candidates}
If $p(\bot) \le 1/2$, then $|L \cap \parity| \le 64$. Similarly, if $p(\bot) \ge 1/2$, then $|L\cap \anti| \le 64$.
\end{lemma}
\begin{proof}
We prove the first half of the lemma, and the second half follows from a similar argument. Pick any function $f\in L\cap \parity$, and pick $d\in\cD$ such that $d(\bot) = 1$ and $d(x) = (-1)^{f(x)}$. Define $p_0$ to be the predictor that maps everything to $1/2$. We have $\langle p_f - p_0, d \rangle_\mu = 1/3$. Define $\mu'$ to be the uniform distribution over $\{0,1\}^d$. We have 
\[
\langle p - p_0,d\rangle_{\mu} = (1/3)( p(\bot) - p_0(\bot)) + (2/3)\langle p - p_0, d\rangle_{\mu'} \le (2/3)\langle p - p_0, d\rangle_{\mu'}. 
\]
Therefore,
\[
1 / 3 - (2/3)\langle p - p_0, d\rangle_{\mu'} \le \langle p_f - p,d\rangle_\mu \le 1/6 + 1/8.
\]
This implies
\begin{equation}
\label{eq:agnostic-inner}
\langle p - p_0, d \rangle_{\mu'} \ge 1/16.
\end{equation}
However, the predictors in $\cD$, when restricted to the sub-domain $\{0,1\}^m\subseteq X$, form an orthonormal basis for $\bR^{\{0,1\}^m}$ w.r.t.\ the inner product $\langle \cdot, \cdot \rangle_{\mu'}$, so
\[
{\sum}_{d\in\cD}\langle p - p_0, d\rangle_{\mu'}^2 = \langle p - p_0, p - p_0\rangle_{\mu'} \le 1/4.
\]
Therefore, there can be at most $64$ different $d\in \cD$ that satisfy \eqref{eq:agnostic-inner}. Since $d$ is defined differently for different $f\in L\cap \parity$, we get $|L\cap \parity| \le 64$ as desired.
\end{proof}

\begin{lemma}
\label{lm:prob-containing}
The event $t\in L$ happens with probability at least $1 - \delta$.
\end{lemma}
\begin{proof}
By the definition of $L$ in \eqref{eq:def-L},
the lemma is trivial if $t\notin \parity\cup \anti$, so we assume $t\in\parity\cup\anti$. 

Define $\mu'$ to be the uniform distribution over $\{0,1\}^m$.
If $t\in\parity$, for the predictor $d\in\cD$ that satisfy $d(\bot) = 1$ and $d(x) = (-1)^{t(x)}$ for all $x\in \{0,1\}^m$, we have $\langle p_t - p_2,d\rangle_{\mu'} = 1/2$, so
\[
\langle p_t - p_2,d \rangle_\mu = (1/3)(-1/2) + (2/3)\langle p_t - p_2,d\rangle_{\mu'} = 1/6.
\]
For all other predictors $d\in\cD$, we have $\langle p_t - p_2,d\rangle_{\mu'} = 0$, so
\[
\langle p_t - p_2,d \rangle_\mu = (1/3)(-1/2) + (2/3)\langle p_t - p_2,d\rangle_{\mu'} = -1/6.
\]
Therefore, $\|p_t - p_2\|_{\mu,\cD} = 1/6$. Similarly, we can show that when $t \in \anti$, $\|p_t - p_1\|_{\mu,\cD} = 1/6$.

Since the input examples $(x_1,o_1),\ldots,(x_n, o_n)$ to algorithm $\cA$ are generated i.i.d.\ from $\mu_{p_t}$, and
since $\cA\in \dsa_n(\cP,\cD,1/8, \delta,\mu)$,
with probability at least $1 - \delta$, algorithm $\cA$ outputs $p$ such that $\|p - p_t\|_{\mu,\cD} \le 1/6 + 1/8$, in which case $t\in L$, as desired.
\end{proof}
Now we complete the proof of \lemmaref{lm:agnostic-lower}.
\begin{proof}
Combining \lemmaref{lm:parity}, \lemmaref{lm:agnostic-candidates}, and \lemmaref{lm:prob-containing}, we have 
\[
1 - \delta \le 1 - (1/2)(1 - 1/2^{m - n})(1 - 64/2^{m - n})
\]
whenever $n \le m - 8$.
Setting $\delta = 1/3$, we get $n \ge m - 20$, which completes the proof. 
\end{proof}

\ifdefined\preprint
\else
\acks{
LH is supported by Moses Charikar’s Simons Investigator award and Omer Reingold's NSF Award IIS-1908774. 
CP is supported by the Simons Foundation Collaboration on the Theory of Algorithmic Fairness.
OR is supported by the Simons Foundation Collaboration on the Theory of Algorithmic Fairness, the Sloan Foundation Grant 2020-13941 and the Simons Foundation investigators award 689988.
}
\fi

\bibliography{alt2022-ref}

\appendix

\section{Failure of Empirical Risk Minimization}
\label{sec:failure-erm}
The following two lemmas show that in certain cases the failure probability of \algorithmref{alg:erm} approaches $1$ (instead of $0$) when the number of examples increases.
\begin{lemma}
Assume $X$ is finite and $\mu$ is the uniform distribution over $X$.
Assume $\cD = [-1,1]^X$, and $\cP$ consists of the following three predictors: $p_b$ that maps every $x\in X$ to $b$ for $b = 0,1/2,1$. For every positive integer $n$, \algorithmref{alg:erm} does not belong to
\[
\oi_n(p_{1/2},\cD, 1/3, 1 - 1/\sqrt{n+1}, \mu).
\]
\end{lemma}
\begin{proof}
Consider the $n$ input examples to \algorithmref{alg:erm}: $(x_1,o_1),\ldots,(x_n,o_n)$ drawn i.i.d.\ from $\mu_{p_{1/2}}$. We first show that with probability above $1 - 1/\sqrt{n+1}$, 
\[
n':=|\{i\in\{1,\ldots,n\}:o_i = 1\}| \ne n/2. 
\]
This is trivially true when $n$ is odd. When $n$ is even, this is also true because
\[
\Pr[n' = n/2] = {n\choose n/2}/2^n < 1/\sqrt{n+1}.
\]

It remains to prove that whenever $n' \ne n/2$, the output $p$ of \algorithmref{alg:erm} satisfies $\|p - p_{1/2}\|_{\mu,\cD} > 1/3$. 
By \lemmaref{lm:advantage-l1}, $\|p_0 - p_{1/2}\|_{\mu,\cD} = 1/2, \|p_{1/2} - p_1\|_{\mu,\cD} = 1/2$, and $\|p_0 - p_1\|_{\mu,\cD} = 1$. Therefore, the $\varepsilon/2(= 1/6)$-covering $\cP'$ in \algorithmref{alg:erm} is equal to $\cP$. The loss of predictor $p_b$ is
\[
\loss(p_b) = \frac 1n \sum_{i=1}^n |b - o_i| = (1 - b)n' + b(n - n') = n' + b(n - 2n').
\]
When $n' < n/2$, $p_0$ has the smallest loss, so \algorithmref{alg:erm} returns $p = p_0$, in which case we indeed have $\|p - p_{1/2}\|_{\mu,\cD} = 1/2 > 1/3$. Similarly, when $n' > n/2$, $p_1$ has the smallest loss, so $p = p_1$ and $\|p - p_{1/2}\|_{\mu,\cD} = 1/2 > 1/3$.
\end{proof}
In the lemma below, 
we give an example of the failure of \algorithmref{alg:erm} in distribution-specific realizable OI when all the predictors in $\cP$ are binary classifiers.
In this example,
$X,\mu,\cP,\cD$ are parametrized by two positive integers $m$ and $n$ as follows.
We choose the individual set to be $X = \{-1,-2,-3\}\cup \{1,\ldots,m\}$, and choose the distinguisher class $\cD$ as follows. For every size-$n$ subset $Y\subseteq \{1,\ldots,m\}$, define $\cD_Y\subseteq [-1,1]^X$ to be the set of all distinguishers $d\in [-1,1]^X$ satisfying $d(x) = 0$ for all $x\in\{1,\ldots,m\}\backslash Y$. The distinguisher class $\cD$ is then defined as $\cD = \bigcup_{Y} \cD_Y$, where the union is over all size-$n$ subsets $Y\subseteq\{1,\ldots,m\}$. The predictor class $\cP$ consists of $4$ predictors $p_0,p_1,p_2,p_3$ defined as follows:
\begin{align*}
p_0(-1) = 0, p_0(-2) = 0, p_0(-3) = 0, {} & p_0(x) = 0 \ \textnormal{for all}\ x\in\{1,\ldots,m\},\\
p_1(-1) = 0, p_1(-2) = 0, p_1(-3) = 0, {} & p_1(x) = 1 \ \textnormal{for all}\ x\in\{1,\ldots,m\},\\
p_2(-1) = 1, p_2(-2) = 1, p_2(-3) = 0, {} & p_2(x) = 0 \ \textnormal{for all}\ x\in\{1,\ldots,m\},\\
p_3(-1) = 0, p_3(-2) = 1, p_3(-3) = 1, {} & p_3(x) = 1 \ \textnormal{for all}\ x\in\{1,\ldots,m\}.
\end{align*}
The distribution $\mu$ spreads $1/2$ probability mass evenly on $\{-1,-2,-3\}$, and the spreads the remaining $1/2$ probability mass evenly on $\{1,\ldots,m\}$.

Since $N_{\mu,\cD}(\cP,\varepsilon/32) \le |\cP| = 4$, \theoremref{thm:characterization} tells us that
\[
\sampdsr(\cP,\cD,\varepsilon,\delta,\mu) \le O(\varepsilon^{-2}\log(2/\delta) + \varepsilon^{-4}).
\]
The lemma below shows that this sample complexity upper bound cannot be achieved using \algorithmref{alg:erm} when $\varepsilon \le 1/4$ and $\delta$ is close to zero.
\begin{lemma}
\label{lm:failure-erm-2}
For every positive integer $n$, there exists a positive integer $m$ such that when $X, \cP,\cD,\mu$ are defined as above, \algorithmref{alg:erm} does not belong to
\[
\dsr_n(\cP,\cD,1/4,\max\{2^{-2-O(n)}, 1 - O(2^{-\Omega(n)})\},\mu).
\]
\end{lemma}
\begin{proof}
By choosing $m$ sufficiently large, we get
\begin{align}
\|p_0 - p_1\|_{\mu,\cD} & = (1/2)(n/m) \le 1/100,\label{eq:failure-distance-1}\\
\|p_i - p_j\|_{\mu,\cD} & \ge (1/2)(2/3) = 1/3 \ \textnormal{for all $i,j\in\{1,2,3,4\}$ satisfying $i < j$ and $(i,j)\ne (0,1)$}.\label{eq:failure-distance-2}
\end{align}

Suppose \algorithmref{alg:erm} belongs to $\dsr_n(\cP,\cD,1/4,\delta,\mu)$ for some $\delta\in(0,1)$. Consider the case where the input to \algorithmref{alg:erm} is $n$ examples $(x_1,o_1),\ldots,(x_n,o_n)$ drawn i.i.d.\ from $\mu_{p^*}$ where $p^*$ is drawn uniformly at random from $\{p_0,p_2\}$. By assumption, the output $p$ of \algorithmref{alg:erm} satisfies $\Pr[\|p - p^*\|_{\mu,\cD}\le 1/4] \ge 1 - \delta$.
Since \algorithmref{alg:erm} only outputs $p\in\cP$, this means that $\Pr[p \in \nei(p^*)] \ge 1- \delta$ where $\nei(p_0) = \{p_0,p_1\}$ and $\nei(p_2) = \{p_2\}$.
However,
with probability $(1/2)^n$, all the $x_i$'s belong to $\{1,\ldots,m\}$, in which case $o_1 = \cdots = o_n = 0$, giving no information about $p^*$. Therefore, $\Pr[p\notin\nei(p^*)] \ge (1/2)^n(1/2)$. This implies
\begin{equation}
\label{eq:failure-binary-1}
\delta \ge (1/2)^n(1/2) > 2^{-2 - O(n)}.
\end{equation}

Inequalities \eqref{eq:failure-distance-1} and \eqref{eq:failure-distance-2} imply that the covering $\cP'$ computed in \algorithmref{alg:erm} is either $\{p_0,p_2,p_3\}$ or $\{p_1,p_2,p_3\}$. Without loss of generality, we assume  $\cP' = \{p_1,p_2,p_3\}$ because the other case can be handled similarly. 
Now consider the case where the input examples $(x_1,o_1),\ldots,(x_n,o_n)$ are drawn i.i.d.\ from $\mu_{p^*}$ with $p^* = p_0$.

By our construction of $\cD$, the loss of every predictor in $p'\in\cP'$ is
\[
\loss(p') = \frac 1n\sum_{i=1}^n |p'(x_i) - p_0(x_i)|.
\]
By the Chernoff bound and the union bound, with probability above $1 - O(2^{-\Omega(n)})$, the absolute difference between $\loss(p')$ and $\bE_{x\sim\mu}|p'(x) - p_0(x)|$ is below $1/100$ for all $p'\in\cP'$.
Note that
\[
\bE_{x\sim\mu}|p_1(x) - p_0(x)| = 1/2, \bE_{x\sim\mu}|p_2(x) - p_0(x)| = 1/3, \bE_{x\sim\mu}|p_3(x) - p_0(x)| = 5/6.
\]
Therefore, with probability above $1 - O(2^{-\Omega(n)})$, \algorithmref{alg:erm} returns $p = p_2$, which does not satisfy $\|p - p^*\|_{\mu,\cD} \le 1/4$. This implies
\begin{equation}
\label{eq:failure-binary-2}
\delta > 1 - O(2^{-\Omega(n)}).
\end{equation}
The lemma is proved by combining \eqref{eq:failure-binary-1} and \eqref{eq:failure-binary-2}.
\end{proof}

\section{A Helper Lemma}
\begin{lemma}
\label{lm:helper}
Let $\mu$ be the uniform distribution over a set $X$ of individuals with size $|X| = n\in\bZ_{>0}$. Then for all $\varepsilon \in (0,1/e)$,
\[
\log N_{\mu, [-1,1]^X}([-1,1]^X, \varepsilon) \ge n\log (1/e\varepsilon),
\]
where $e$ is the base of the natural logarithm.
\end{lemma}
\begin{proof}
Suppose $X = \{x_1,\ldots,x_n\}$. Let $\vect$ be the bijection from $\bR^X$ to $\bR^n$ such that $\vect(f) = (f(x_1),\ldots,f(x_n))\in\bR^n$ for all $f\in\bR^X$.
Let $\{f_1,\ldots,f_m\}\subseteq [-1,1]^X$ be an $\varepsilon$-covering of $[-1,1]^X$ w.r.t.\ the norm $\|\cdot\|_{\mu,[-1,1]^X}$. Defining
\[
B = \{f\in \bR^X:\|f\|_{\mu,[-1,1]^X} \le \varepsilon\},
\]
we have $[-1,1]^X\subseteq \bigcup_{i=1}^m (\{f_i\} + B)$,
which implies $\vect([-1,1]^X) \subseteq \bigcup_{i=1}^m \vect(\{f_i\} + B)$.
Therefore, the volume of $\vect([-1,1]^X)$ is at most $m$ times the volume of $\vect(B)$.

It is clear that $\vect([-1,1]^X) = [-1,1]^n$ has volume $2^n$ in $\bR^n$. Moreover, by \lemmaref{lm:advantage-l1},
\[
\vect(B) = \{(r_1,\ldots,r_n)\in\bR^n: |r_1| + \cdots + |r_n| \le \varepsilon n\},
\]
which has volume $(2\varepsilon n)^n/n!$. 
Therefore,
\[
2^n \le m(2\varepsilon n)^n/n!,
\]
and thus
\[
\log m \ge \log (n!/(n\varepsilon)^n) \ge n \log (1/e\varepsilon),
\]
as desired. We used the fact $\log (n!) \ge \int_1^n \log t \mathrm{d}t > n\log (n/e)$ in the last inequality.
\end{proof}
\end{document}